\documentclass[journal]{IEEEtran}

\usepackage{amsmath,amssymb,amsfonts}
\usepackage{algorithmic}
\usepackage{algorithm}
\usepackage{graphicx}
\usepackage{textcomp}
\usepackage{xcolor}
\usepackage{cite}
\usepackage{hyperref}
\usepackage{booktabs}
\usepackage{multirow}
\usepackage{bm}
\usepackage{amsthm}

\newtheorem{theorem}{Theorem}
\newtheorem{proposition}[theorem]{Proposition}
\newtheorem{corollary}[theorem]{Corollary}

\theoremstyle{definition}
\newtheorem{definition}[theorem]{Definition}
\newtheorem{assumption}[theorem]{Assumption}

\theoremstyle{remark}
\newtheorem{remark}[theorem]{Remark}
\newcommand{\E}{\mathbb{E}}
\newcommand{\R}{\mathbb{R}}
\newcommand{\x}{\mathbf{x}}

\newcommand{\qL}{q_L}
\newcommand{\qLhat}{\hat{q}_L}
\newcommand{\Qtilde}{\tilde{Q}}
\newcommand{\Qp}{Q_P}
\newcommand{\Qc}{Q_C}

\begin{document}

\title{Online Bayesian Imbalanced Learning with Bregman-Calibrated Deep Networks}

\author{
\IEEEauthorblockN{Zahir Alsulaimawi,~\IEEEmembership{Member,~IEEE}}
\IEEEauthorblockA{School of Electrical Engineering and Computer Science\\
Oregon State University, Corvallis, OR, USA\\
alsulaiz@oregonstate.edu}
}

\maketitle

\begin{abstract}
Class imbalance remains a fundamental challenge in machine learning, where standard classifiers exhibit severe performance degradation in minority classes. Although existing approaches address imbalance through resampling or cost-sensitive learning during training, they require retraining or access to labeled target data when class distributions shift at deployment time, a common occurrence in real-world applications such as fraud detection, medical diagnosis, and anomaly detection. We present \textit{Online Bayesian Imbalanced Learning} (OBIL), a principled framework that decouples likelihood-ratio estimation from class-prior assumptions, enabling real-time adaptation to distribution shifts without model retraining. Our approach builds on the established connection between Bregman divergences and proper scoring rules to show that deep networks trained with such losses produce posterior probability estimates from which prior-invariant likelihood ratios can be extracted. We prove that these likelihood-ratio estimates remain valid under arbitrary changes in class priors and cost structures, requiring only a threshold adjustment for optimal Bayes decisions. We derive finite-sample regret bounds demonstrating that OBIL achieves $O(\sqrt{T \log T})$ regret against an oracle with perfect prior knowledge. Extensive experiments on benchmark datasets and medical diagnosis
benchmarks under simulated deployment shifts demonstrate that OBIL maintains robust performance under severe distribution shifts, outperforming state-of-the-art methods in F1 Score when test distributions deviate significantly from the training conditions.
\end{abstract}

\begin{IEEEkeywords}
Imbalanced learning, Bayesian decision theory, likelihood ratio estimation, distribution shift, Bregman divergences, online learning, deep neural networks.
\end{IEEEkeywords}

%=====================================================================
\section{Introduction}
\label{sec:introduction}
%=====================================================================

\IEEEPARstart{C}{lass} imbalance constitutes one of the most pervasive and challenging problems in machine learning, arising naturally in applications where the phenomenon of interest is inherently rare \cite{he2009learning, krawczyk2016learning}. In fraud detection systems, legitimate transactions outnumber fraudulent ones by factors exceeding 1000:1; in medical diagnosis, healthy patients vastly outnumber those with rare diseases; in anomaly detection, normal operating conditions dominate over failure states. Standard classification algorithms, designed to minimize overall error rate, systematically fail in these scenarios by developing strong biases toward majority classes---achieving high accuracy while catastrophically misclassifying the minority instances that often carry the greatest practical significance \cite{chawla2004special}.

The machine learning community has developed extensive methodologies to address class imbalance, broadly categorized into data-level and algorithm-level approaches \cite{galar2012review, fernandez2018learning}. Data-level methods manipulate the training distribution through undersampling majority instances \cite{liu2009exploratory}, oversampling minority instances via techniques such as SMOTE \cite{chawla2002smote} and its variants \cite{he2008adasyn, han2005borderline}, or hybrid combinations thereof. Algorithm-level approaches modify the learning objective through cost-sensitive weighting \cite{elkan2001foundations, zhou2006training}, threshold adjustment \cite{maloof2003learning}, or ensemble methods that integrate multiple rebalancing strategies \cite{seiffert2010rusboost, liu2009exploratory}. While these techniques have achieved considerable success, they share a fundamental limitation: \textit{they assume that the class distribution encountered during deployment matches that observed during training}.

This assumption proves untenable in numerous real-world scenarios. Consider a medical diagnostic system trained on data from a tertiary care hospital, where disease prevalence is elevated due to referral patterns, subsequently deployed in a primary care setting where the same conditions are substantially rarer. The classifier, calibrated for the training distribution, will exhibit degraded specificity as its decision threshold no longer corresponds to optimal operation under the deployment prior. Similarly, fraud detection systems face adversarial drift as attack patterns evolve, credit scoring models encounter demographic shifts across regions, and anomaly detectors must adapt to changing baseline conditions in industrial processes \cite{gama2014survey, lu2018learning}.

Current approaches to this \textit{distribution shift} problem typically require model retraining or recalibration on labeled data from the target distribution---a requirement that is often impractical due to the expense of annotation, the delay inherent in collecting sufficient samples, or the need for immediate deployment. What is needed is a principled framework that can adapt to changing class priors \textit{without} accessing labeled target data and \textit{without} retraining the underlying model.

The invariance of likelihood ratios to class priors is a classical result in Bayesian decision theory \cite{duda2012pattern}. However, exploiting this invariance in modern deep learning systems requires addressing three open questions: (1)~under what conditions do deep networks produce the calibrated posterior estimates needed for valid likelihood ratio extraction? (2)~how can the deployment-time prior be estimated online without labeled data? and (3)~what performance guarantees can be provided for such an adaptive system? This paper provides answers to all three.

\subsection{The Bayesian Perspective: Likelihood Ratios as Invariant Quantities}

Bayesian decision theory provides an elegant solution to this challenge through a fundamental observation: while posterior probabilities $P(y|\x)$ depend on class priors, the \textit{likelihood ratio}
\begin{equation}
    \qL(\x) = \frac{p(\x|y=1)}{p(\x|y=0)}
    \label{eq:likelihood_ratio}
\end{equation}
is determined solely by the class-conditional densities and is therefore \textit{invariant} to changes in prior probabilities and misclassification costs \cite{duda2012pattern}. This invariance property suggests a two-stage approach to imbalanced classification: first, estimate the likelihood ratio from training data under favorable conditions (e.g., after rebalancing); second, combine this estimate with the deployment-time prior to obtain optimal decisions.

The optimal Bayes decision rule for binary classification under costs $C_{ij}$ (cost of deciding class $i$ when true class is $j$) and priors $P_0, P_1$ takes the form
\begin{equation}
    \qL(\x) \underset{y=0}{\overset{y=1}{\gtrless}} \Qc \cdot \Qp = Q
    \label{eq:bayes_rule_1}
\end{equation}
where $\Qc = (C_{10} - C_{00})/(C_{01} - C_{11})$ denotes the cost ratio and $\Qp = P_0/P_1$ denotes the imbalance ratio. Crucially, when class priors change from training ($\Qp^{\text{train}}$) to deployment ($\Qp^{\text{deploy}}$), only the threshold $Q$ requires adjustment---the likelihood ratio estimate $\qLhat(\x)$ remains valid.

This observation, while classical, has received surprisingly limited attention in the modern deep learning literature on imbalanced classification. Roca-Sotelo \textit{et al.} \cite{roca2016exploratory} provided foundational empirical evidence that likelihood ratios estimated from artificially balanced ``associated'' problems can solve the original imbalanced problem using shallow MLPs. Our work builds directly on their conceptual framework while extending it to modern deep architectures with theoretical guarantees, online adaptation, and regret analysis; we provide a detailed comparison in Section~\ref{sec:bayesian_related}.

\subsection{Bregman Divergences and Posterior Probability Estimation}

A key enabler of the likelihood ratio approach is the remarkable property of certain loss functions to yield outputs that directly estimate posterior probabilities. This property traces to Bregman divergence theory \cite{bregman1967relaxation}, which establishes that loss functions satisfying
\begin{equation}
    \frac{\partial C}{\partial o} = -g(o)(t - o), \quad g(o) > 0
    \label{eq:bregman_condition}
\end{equation}
produce optimal predictors $o^*(\x) = \E[t|\x]$, where $t \in \{-1, +1\}$ is the class label \cite{cid1999cost, cid2001structure}. For binary classification, this implies $o^*(\x) = 2P(y=1|\x) - 1$, enabling direct extraction of posterior probabilities from network outputs.

Common loss functions satisfying this property include:
\begin{itemize}
    \item \textbf{Squared error}: $C(o, t) = \frac{1}{2}(t - o)^2$, yielding $g(o) = 1$
    \item \textbf{Cross-entropy} (with appropriate parameterization): yielding $g(o) = (1 - o^2)^{-1}$
    \item \textbf{Exponential loss}: $C(o, t) = \exp(-to)$. The Bregman condition is verified at the optimal point $o^* = \E[t|\x]$; the resulting $g$ function depends on $o$ alone through $g(o) = \cosh(\text{arctanh}(o))^{-1}$ after reparameterization
\end{itemize}

This property is \textit{architecture-agnostic}: whether the predictor is a shallow network, a deep convolutional network, a transformer, or a graph neural network, the Bregman condition ensures posterior probability estimation at the output. This observation opens the door to leveraging the representational power of modern deep architectures while maintaining the principled Bayesian framework for handling imbalance.

\subsection{Contributions}

In this paper, we develop \textit{Online Bayesian Imbalanced Learning} (OBIL), a comprehensive framework that unifies Bayesian decision theory, Bregman divergence properties, and online learning to address class imbalance under distribution shift. Our contributions are as follows:

\begin{enumerate}
    \item \textbf{Theoretical Framework}: We extend the classical Bregman-proper scoring rule connection \cite{cid1999cost, cid2001structure} to modern deep architectures, proving that likelihood ratio estimates extracted from Bregman-calibrated networks are invariant to the training distribution's class balance under a precisely stated capacity condition (Section~\ref{sec:theory}).
    
    \item \textbf{Online Adaptation Algorithm}: We develop OBIL, an algorithm that maintains a likelihood ratio estimator while continuously adapting its decision threshold based on streaming estimates of the deployment class prior. We prove that OBIL requires no labeled data from the target distribution for threshold adaptation, relying instead on confident pseudo-labels derived from the likelihood ratio estimator. We characterize conditions under which this self-referential estimation procedure converges (Section~\ref{sec:algorithm}).
    
    \item \textbf{Regret Analysis}: We derive finite-sample regret bounds for OBIL, proving that under smoothly varying priors and a Tsybakov margin condition, our method achieves $O(\sqrt{T \log T})$ regret relative to an oracle with perfect prior knowledge. We identify conditions under which this bound is tight (Section~\ref{sec:regret}).
    
    \item \textbf{Bregman-Calibrated Deep Networks}: We provide practical guidelines for implementing Bregman-calibrated networks across modern architectures, including loss function selection, output layer design, and calibration verification procedures linked to our error propagation analysis (Section~\ref{sec:implementation}).
    
    \item \textbf{Empirical Validation}: We conduct extensive experiments on benchmark imbalanced datasets and simulated distribution shift scenarios, demonstrating that OBIL maintains robust performance under severe distribution shift while existing methods degrade. We also present failure mode analysis characterizing when OBIL breaks down (Section~\ref{sec:experiments}).
\end{enumerate}

\subsection{Paper Organization}

The remainder of this paper is organized as follows. Section~\ref{sec:related} reviews related work on imbalanced learning, distribution shift adaptation, and likelihood ratio estimation. Section~\ref{sec:preliminaries} establishes notation and reviews the necessary background on Bayesian decision theory and Bregman divergences. Section~\ref{sec:theory} develops the theoretical foundation for likelihood ratio estimation via Bregman-calibrated networks. Section~\ref{sec:algorithm} presents the OBIL algorithm and its variants. Section~\ref{sec:regret} provides regret analysis and theoretical guarantees. Section~\ref{sec:implementation} discusses practical implementation considerations. Section~\ref{sec:experiments} presents experimental results. Section~\ref{sec:conclusion} concludes with discussion and future directions.

%=====================================================================
\section{Related Work}
\label{sec:related}
%=====================================================================

\subsection{Class Imbalance in Machine Learning}

The class imbalance problem has motivated extensive research spanning several decades \cite{japkowicz2002class, he2009learning}. Early work recognized that standard classifiers optimize accuracy, which is dominated by majority class performance when classes are imbalanced \cite{provost2000machine}. This insight led to the development of alternative evaluation metrics---precision, recall, F-measure, area under the ROC curve (AUC), and area under the precision-recall curve (AUPRC)---that better capture minority class performance \cite{davis2006relationship, saito2015precision}.

\subsubsection{Data-Level Methods}

Data-level approaches modify the training distribution to reduce imbalance. Random undersampling removes majority instances but risks discarding informative examples \cite{drummond2003c4}. Random oversampling duplicates minority instances but may cause overfitting \cite{chawla2002smote}. SMOTE \cite{chawla2002smote} addresses overfitting by generating synthetic minority examples through interpolation in feature space. Subsequent variants improve upon SMOTE by focusing generation on borderline instances (Borderline-SMOTE \cite{han2005borderline}), adapting to local density (ADASYN \cite{he2008adasyn}), or combining generation with informed undersampling (SMOTE-ENN, SMOTE-Tomek \cite{batista2004study}).

\subsubsection{Algorithm-Level Methods}

Algorithm-level approaches modify the learning procedure itself. Cost-sensitive learning assigns higher misclassification costs to minority errors, effectively reweighting the loss function \cite{elkan2001foundations, zhou2006training}. This approach is theoretically grounded in Bayesian decision theory but requires specification of cost ratios, which may be unknown in practice. Threshold adjustment methods train standard classifiers but optimize the decision threshold for imbalanced metrics post-hoc \cite{maloof2003learning}. One-class classification learns a model of the minority class only, treating majority instances as ``outliers'' \cite{khan2014one}.

\subsubsection{Ensemble Methods}

Ensemble approaches combine multiple classifiers trained under different rebalancing conditions. Bagging-based methods such as UnderBagging \cite{barandela2003strategies} and OverBagging \cite{wang2009diversity} create diverse base learners through varied sampling. Boosting-based methods such as SMOTEBoost \cite{chawla2003smoteboost} and RUSBoost \cite{seiffert2010rusboost} integrate resampling with adaptive boosting. Recent work explores how ensemble diversity relates to imbalanced classification performance \cite{galar2012review}.

\subsection{Distribution Shift and Domain Adaptation}

Distribution shift occurs when training and test distributions differ \cite{quinonero2009dataset}. In the context of class imbalance, \textit{prior probability shift} (also called \textit{target shift} or \textit{label shift}) refers specifically to changes in class priors $P(y)$ while class-conditional distributions $P(\x|y)$ remain stable \cite{storkey2009training, lipton2018detecting}.

Classical approaches to label shift require labeled target data for importance reweighting \cite{shimodaira2000improving}. Recent work develops methods for detecting and correcting label shift using only unlabeled target data by exploiting the relationship between source and target label distributions \cite{lipton2018detecting, azizzadenesheli2019regularized, garg2020unified}. However, these methods typically assume access to a batch of target samples and do not address the online setting where data arrives sequentially and the prior may change continuously.

Domain adaptation more broadly addresses distribution shift across domains \cite{ben2010theory, ganin2016domain}. While powerful, most domain adaptation methods assume access to (unlabeled) target domain data during training, which may not be available when deployment conditions are unknown a priori.

\subsection{Likelihood Ratio Estimation}

Direct estimation of likelihood ratios (or equivalently, density ratios) has a rich history in statistics and machine learning \cite{sugiyama2012density}. Classical approaches estimate individual densities and compute their ratio, but this two-step procedure can amplify estimation errors. Direct density ratio estimation methods avoid this issue by estimating the ratio directly.

Kernel-based methods include Kernel Mean Matching (KMM) \cite{huang2006correcting}, Kullback-Leibler Importance Estimation Procedure (KLIEP) \cite{sugiyama2008direct}, and unconstrained Least-Squares Importance Fitting (uLSIF) \cite{kanamori2009least}. These methods minimize divergence measures between weighted source and target distributions.

Neural network approaches to density ratio estimation have emerged recently \cite{rhodes2020telescoping, choi2022density}, leveraging the expressive power of deep networks. These methods typically frame density ratio estimation as classification between samples from two distributions, exploiting the connection between optimal classifiers and likelihood ratios \cite{mohamed2016learning}.

Our work differs from existing density ratio estimation methods in several key aspects: (1)~we exploit the Bregman divergence structure to obtain calibrated posterior probability estimates from which likelihood ratios are derived; (2)~we develop an online adaptation procedure that does not require target domain samples; (3)~we provide regret guarantees specific to the imbalanced classification setting.

\subsection{Bayesian Approaches to Imbalanced Learning}
 \label{sec:bayesian_related}

Bayesian perspectives on imbalanced learning have been explored but remain underutilized. Roca-Sotelo \textit{et al.} \cite{roca2016exploratory} introduced the concept of ``associated problems''---artificially balanced versions of imbalanced problems---and showed that likelihood ratios estimated from associated problems transfer to the original problem. Their work provided the foundational insight that likelihood ratios from associated problems transfer to the original problem. However, it was limited to shallow MLPs with a single hidden layer, did not consider online adaptation to changing priors, and provided empirical rather than theoretical analysis of estimation quality. Our work builds directly on their conceptual framework while extending it significantly in scope, theory, and practical applicability.

Bayesian neural networks offer an alternative approach by maintaining uncertainty estimates over predictions \cite{gal2016dropout, blundell2015weight}. While principled, these methods do not directly address the prior shift problem and typically require expensive inference procedures.

\subsection{Logit Adjustment and Post-Hoc Correction}

Recent work on long-tail recognition has developed post-hoc logit adjustment methods that share our Bayesian motivation. Menon \textit{et al.} \cite{menon2021longtail} showed that adjusting logits by log-prior ratios yields Bayes-optimal classifiers under label shift, which is equivalent to threshold adjustment in the binary case. Cao \textit{et al.} \cite{cao2019ldam} proposed label-distribution-aware margin loss for long-tail learning. These methods adjust for known prior shifts but do not address the online setting where priors are unknown and changing.

Post-hoc calibration methods beyond temperature scaling, including Platt scaling \cite{platt1999probabilistic} and isotonic regression \cite{zadrozny2002transforming}, can improve posterior estimates but do not address prior shift adaptation. Our approach complements these methods: calibration techniques improve the quality of posterior estimates (and hence likelihood ratios), while OBIL adapts to prior shift at deployment time.

%=====================================================================
\section{Preliminaries}
\label{sec:preliminaries}
%=====================================================================

\subsection{Notation}

We consider binary classification with input space $\mathcal{X} \subseteq \R^d$ and label space $\mathcal{Y} = \{0, 1\}$. We use $y \in \{0,1\}$ for classification labels and $t = 2y - 1 \in \{-1,+1\}$ as the corresponding target for Bregman loss computation. All subsequent equations use $y$ unless explicitly noted as using $t$. Let $P_0 = P(y=0)$ and $P_1 = P(y=1) = 1 - P_0$ denote the class priors, and let $p(\x|y=i)$ denote the class-conditional density for class $i$. We define:
\begin{itemize}
    \item \textbf{Imbalance ratio}: $\Qp = P_0/P_1$ (ratio of majority to minority prior)
    \item \textbf{Cost ratio}: $\Qc = (C_{10} - C_{00})/(C_{01} - C_{11})$ where $C_{ij}$ is the cost of predicting class $i$ when true class is $j$
    \item \textbf{Likelihood ratio}: $\qL(\x) = p(\x|y=1)/p(\x|y=0)$
    \item \textbf{Combined threshold}: $Q = \Qc \cdot \Qp$
\end{itemize}

We use the following probability notation consistently throughout: $P(y=1|\x)$ denotes the true posterior under the original problem, $\tilde{P}(y=1|\x)$ denotes the true posterior under an associated problem, and $\hat{P}(y=1|\x)$ denotes an estimated posterior. We write $\Pr(\cdot)$ for general probability statements not conditioned on $\x$.

For a classifier $f: \mathcal{X} \to \R$ with decision rule $\hat{y} = \mathbb{1}[f(\x) > \tau]$, we denote the posterior probability estimate as $\hat{P}(y=1|\x)$ and the likelihood ratio estimate as $\qLhat(\x)$.

\subsection{Bayesian Decision Theory for Classification}

The Bayes-optimal classifier minimizes expected cost:
\begin{equation}
    \hat{y}^* = \arg\min_{j \in \{0,1\}} \sum_{i=0}^{1} C_{ji} P(y=i|\x)
\end{equation}

For binary classification with arbitrary costs, this reduces to deciding $\hat{y}=1$ when:
\begin{equation}
    \frac{P(y=1|\x)}{P(y=0|\x)} > \frac{C_{10} - C_{00}}{C_{01} - C_{11}} = \Qc
    \label{eq:posterior_rule}
\end{equation}

Using Bayes' theorem to express posteriors in terms of likelihoods:
\begin{equation}
    P(y=1|\x) = \frac{p(\x|y=1)P_1}{p(\x|y=1)P_1 + p(\x|y=0)P_0}
    \label{eq:bayes_posterior}
\end{equation}

Substituting into \eqref{eq:posterior_rule} and simplifying yields the likelihood ratio rule:
\begin{equation}
    \qL(\x) = \frac{p(\x|y=1)}{p(\x|y=0)} \underset{y=0}{\overset{y=1}{\gtrless}} \Qc \cdot \Qp = Q
    \label{eq:lr_rule}
\end{equation}

\textbf{Observation (Prior Invariance).} The likelihood ratio $\qL(\x) = p(\x|y=1)/p(\x|y=0)$ depends only on the class-conditional densities and is therefore invariant to changes in class priors and misclassification costs. This classical observation \cite{duda2012pattern} is the foundation of our approach: if we can estimate $\qL(\x)$ under favorable training conditions (e.g., balanced classes), the estimate remains valid when deployed under different priors.

\subsection{Bregman Divergences and Proper Scoring Rules}

\begin{definition}[Bregman Divergence]
Let $\phi: \Omega \to \R$ be a strictly convex, differentiable function. The Bregman divergence associated with $\phi$ is:
\begin{equation}
    D_\phi(p \| q) = \phi(p) - \phi(q) - \nabla\phi(q)^\top(p - q)
\end{equation}
\end{definition}

\begin{definition}[Proper Scoring Rule]
A loss function $C(o, t)$ is a \textit{proper scoring rule} if, for any distribution over $t$, the expected loss is minimized when $o$ equals the expected value of $t$:
\begin{equation}
    \arg\min_o \E_t[C(o, t)] = \E[t]
\end{equation}
\end{definition}

The connection between Bregman divergences and proper scoring rules is established by the following result:

\begin{theorem}[Savage, 1971; Cid-Sueiro et al., 1999]
\label{thm:bregman_proper}
A loss function $C(o, t)$ is a proper scoring rule if and only if it can be written as a Bregman divergence (plus terms independent of $o$). Equivalently, $C$ is proper if and only if:
\begin{equation}
    \frac{\partial C}{\partial o} = -g(o)(t - o)
\end{equation}
for some function $g(o) > 0$.
\end{theorem}

For binary classification with $t \in \{-1, +1\}$, a network trained to minimize a proper scoring rule produces outputs satisfying:
\begin{equation}
    o^*(\x) = \E[t|\x] = (+1)P(t=+1|\x) + (-1)P(t=-1|\x) = 2P(y=1|\x) - 1
\end{equation}

This enables direct extraction of posterior probabilities:
\begin{equation}
    \hat{P}(y=1|\x) = \frac{o(\x) + 1}{2}
    \label{eq:posterior_from_output}
\end{equation}

\subsection{From Posteriors to Likelihood Ratios}

Given a posterior probability estimate $\hat{P} = \hat{P}(y=1|\x)$ obtained from a network trained under imbalance ratio $\Qtilde_P$, we can recover the likelihood ratio using the inverse of \eqref{eq:bayes_posterior}:
\begin{equation}
    \qLhat(\x) = \Qtilde_P \cdot \frac{\hat{P}}{1 - \hat{P}}
    \label{eq:lr_from_posterior}
\end{equation}

Combining with \eqref{eq:posterior_from_output}:
\begin{equation}
    \qLhat(\x) = \Qtilde_P \cdot \frac{1 + o(\x)}{1 - o(\x)}
    \label{eq:lr_from_output}
\end{equation}

This likelihood ratio estimate, obtained under training conditions $\Qtilde_P$, can be applied to make decisions under deployment conditions $\Qp$ using rule \eqref{eq:lr_rule}.

%=====================================================================
\section{Theoretical Framework}
\label{sec:theory}
%=====================================================================

In this section, we establish the theoretical foundations for likelihood ratio estimation via Bregman-calibrated deep networks. We prove that networks trained with proper scoring rules produce valid likelihood ratio estimates that are invariant to training class balance, and we characterize the conditions under which these estimates transfer to deployment scenarios with different class priors.

\subsection{Likelihood Ratio Estimation from Associated Problems}

The central challenge in imbalanced classification is that standard training procedures produce biased estimators when class frequencies are severely skewed. Our key insight is to transform the original \textit{hard} problem into an \textit{associated} problem that is easier to solve, estimate the likelihood ratio from this associated problem, and then apply the estimate to the original problem.

\begin{definition}[Associated Problem]
\label{def:associated}
Given an original classification problem with class priors $(P_0, P_1)$ and cost ratio $\Qc$, an \textit{associated problem} is defined by modified priors $(\tilde{P}_0, \tilde{P}_1)$ and cost ratio $\Qtilde_C$ such that:
\begin{enumerate}
    \item The class-conditional densities $p(\x|y=0)$ and $p(\x|y=1)$ remain unchanged
    \item The modified imbalance ratio $\Qtilde_P = \tilde{P}_0/\tilde{P}_1$ is more favorable for learning (typically closer to 1)
\end{enumerate}
\end{definition}

The following theorem establishes that likelihood ratio estimates from associated problems are valid for the original problem:

\begin{theorem}[Transferability of Likelihood Ratio Estimates]
\label{thm:transfer}
Let $\qLhat(\x)$ be an estimate of the likelihood ratio obtained from an associated problem with imbalance ratio $\Qtilde_P$. If the estimate satisfies
\begin{equation}
    \qLhat(\x) = \Qtilde_P \cdot \frac{\tilde{P}(y=1|\x)}{1 - \tilde{P}(y=1|\x)}
    \label{eq:lr_estimate}
\end{equation}
where $\tilde{P}(y=1|\x)$ is the true posterior probability under the associated problem, then $\qLhat(\x) = \qL(\x)$, the true likelihood ratio, regardless of $\Qtilde_P$.
\end{theorem}

\begin{proof}
Under the associated problem with priors $(\tilde{P}_0, \tilde{P}_1)$, Bayes' theorem gives:
\begin{equation}
    \tilde{P}(y=1|\x) = \frac{p(\x|y=1)\tilde{P}_1}{p(\x|y=1)\tilde{P}_1 + p(\x|y=0)\tilde{P}_0}
\end{equation}

Substituting into \eqref{eq:lr_estimate}:
\begin{align}
    \qLhat(\x) &= \Qtilde_P \cdot \frac{\tilde{P}(y=1|\x)}{1 - \tilde{P}(y=1|\x)} \\
    &= \frac{\tilde{P}_0}{\tilde{P}_1} \cdot \frac{p(\x|y=1)\tilde{P}_1}{p(\x|y=0)\tilde{P}_0} \\
    &= \frac{p(\x|y=1)}{p(\x|y=0)} = \qL(\x)
\end{align}
The $\tilde{P}_0$ and $\tilde{P}_1$ terms cancel exactly, leaving only the ratio of class-conditional densities.
\end{proof}

\subsection{Bregman-Calibrated Networks}

We now establish conditions under which neural networks produce the posterior probability estimates required by Theorem~\ref{thm:transfer}.

\begin{definition}[Bregman-Calibrated Network]
\label{def:bregman_calibrated}
A neural network $f_\theta: \mathcal{X} \to (-1, 1)$ is \textit{Bregman-calibrated} if it is trained to minimize a loss function $\mathcal{L}(\theta)$ such that at the optimum:
\begin{equation}
    f_{\theta^*}(\x) = \E[t|\x] = 2P(y=1|\x) - 1
\end{equation}
for all $\x \in \mathcal{X}$, where $t = 2y-1 \in \{-1, +1\}$ is the label encoding.
\end{definition}

\begin{theorem}[Characterization of Bregman-Calibrated Losses]
\label{thm:bregman_char}
A loss function $C(o, t)$ produces a Bregman-calibrated network if and only if it satisfies:
\begin{equation}
    \frac{\partial C(o, t)}{\partial o} = -g(o)(t - o)
    \label{eq:bregman_condition_full}
\end{equation}
for some continuous function $g: (-1, 1) \to \R^+$ with $g(o) > 0$ for all $o \in (-1, 1)$.
\end{theorem}

\begin{proof}
The expected loss for a predictor $o(\x)$ is:
\begin{equation}
    \mathcal{L}[o] = \E_{\x}\left[\E_{t|\x}[C(o(\x), t)]\right]
\end{equation}

Taking the functional derivative and setting to zero:
\begin{align}
    \frac{\delta \mathcal{L}}{\delta o}(\x) &= \E_{t|\x}\left[\frac{\partial C}{\partial o}\right] = 0 \\
    &= -g(o(\x)) \E_{t|\x}[t - o(\x)] = 0 \\
    &= -g(o(\x)) \left(\E[t|\x] - o(\x)\right) = 0
\end{align}

Since $g(o) > 0$, we must have $o^*(\x) = \E[t|\x]$.

For the converse, if $C$ does not satisfy \eqref{eq:bregman_condition_full}, then the first-order condition yields $o^* \neq \E[t|\x]$ in general.
\end{proof}

\begin{corollary}[Architecture Independence]
\label{cor:architecture}
Let $\mathcal{F}_\Theta$ denote the function class realized by a neural network architecture with parameter space $\Theta$. The Bregman calibration property (Definition~\ref{def:bregman_calibrated}) depends only on the loss function, not on the network architecture, provided the following \emph{capacity condition} is satisfied:
\begin{equation}
    \exists\, \theta^* \in \Theta \text{ such that } 
    f_{\theta^*}(\x) = 2P(y=1|\x) - 1 
    \quad \forall\, \x \in \mathcal{X}.
    \label{eq:capacity_condition}
\end{equation}
That is, the architecture must be expressive enough that the true posterior mapping $\x \mapsto 2P(y=1|\x)-1$ lies within $\mathcal{F}_\Theta$. Under this condition, any architecture (MLP, CNN, Transformer, GNN) trained with a Bregman-type loss will produce exact posterior probability estimates at convergence. When the capacity condition holds only approximately---i.e., $\min_{\theta \in \Theta} \|f_\theta - (2P(y=1|\cdot)-1)\|_\infty \leq \epsilon_{approx}$---the posterior estimation error is bounded by $\epsilon_{approx}$, which propagates to the likelihood ratio estimate through Theorem~\ref{thm:error_prop}.
\end{corollary}

\begin{remark}
In practice, deep networks are universal approximators \cite{hornik1989multilayer}: for any $\epsilon_{approx} > 0$, there exists a network of sufficient width/depth satisfying the approximate capacity condition. However, the required network size depends on the complexity of the posterior mapping and may be exponential in dimension for certain problem classes. The approximation error contributes to the $\epsilon_{LR}$ term in our regret analysis (Theorem~\ref{thm:regret}), and practical calibration verification (Section~\ref{sec:implementation}) is essential to ensure this error is controlled.
\end{remark}

\subsection{Common Bregman-Type Loss Functions}

We catalog several loss functions satisfying the Bregman condition:

\begin{proposition}[Squared Error Loss]
\label{prop:squared}
The squared error loss $C(o, t) = \frac{1}{2}(t - o)^2$ satisfies \eqref{eq:bregman_condition_full} with $g(o) = 1$.
\end{proposition}

\begin{proposition}[Logistic Loss]
\label{prop:logistic}
The logistic loss, when parameterized as
\begin{equation}
    C(o, t) = \log\left(1 + \exp(-t \cdot \text{arctanh}(o))\right)
\end{equation}
satisfies \eqref{eq:bregman_condition_full} with $g(o) = (1 - o^2)^{-1}$.
\end{proposition}

\begin{remark}
In practice, the standard cross-entropy loss applied to sigmoid outputs produces well-calibrated probability estimates under appropriate conditions, though it does not exactly satisfy the Bregman form. We address calibration verification in Section~\ref{sec:implementation}.
\end{remark}

\subsection{Estimation Error Analysis}

In practice, we obtain estimates $\hat{P}(y=1|\x)$ rather than true posteriors. We now characterize how estimation errors propagate to likelihood ratio estimates.

\begin{theorem}[Error Propagation in Likelihood Ratio Estimation]
\label{thm:error_prop}
Let $\hat{P} = \hat{P}(y=1|\x)$ be an estimate of the true posterior $P^* = P(y=1|\x)$ with error $\epsilon = \hat{P} - P^*$. The resulting likelihood ratio estimate $\qLhat(\x)$ satisfies:
\begin{equation}
    \left|\frac{\qLhat(\x) - \qL(\x)}{\qL(\x)}\right| \leq \frac{|\epsilon|}{P^*(1 - P^*) - |\epsilon|(1 - 2P^*)}
    \label{eq:relative_error}
\end{equation}
provided $|\epsilon| < \min(P^*, 1 - P^*)$.
\end{theorem}

\begin{proof}
The true likelihood ratio is $\qL = \Qp \frac{P^*}{1-P^*}$ and the estimate is $\qLhat = \Qp \frac{\hat{P}}{1-\hat{P}}$. The relative error is:
\begin{align}
    \frac{\qLhat - \qL}{\qL} &= \frac{\hat{P}(1-P^*) - P^*(1-\hat{P})}{P^*(1-\hat{P})} \\
    &= \frac{\hat{P} - P^*}{P^*(1-\hat{P})} = \frac{\epsilon}{P^*(1-P^*-\epsilon)}
\end{align}
Taking absolute values and using $|\epsilon| < 1 - P^*$ yields \eqref{eq:relative_error}.
\end{proof}

\begin{corollary}[Sensitivity Near Decision Boundaries]
\label{cor:sensitivity}
The relative error in $\qLhat$ is amplified when $P^*$ is close to 0 or 1 (i.e., near certain classification regions) and minimized when $P^* \approx 0.5$ (near the decision boundary under balanced priors). This has important implications: likelihood ratio estimates are most reliable in the regions where classification is most uncertain, which is precisely where accurate estimates matter most for decision-making.
\end{corollary}

\subsection{Effect of Rebalancing on Likelihood Estimation}

When creating associated problems through resampling, the empirical class-conditional distributions may be distorted. We characterize conditions under which this distortion is benign.

\begin{assumption}[Likelihood-Preserving Resampling]
\label{assump:preserving}
A resampling procedure is \textit{likelihood-preserving} if the resampled class-conditional distributions $\tilde{p}(\x|y=i)$ satisfy:
\begin{equation}
    \frac{\tilde{p}(\x|y=1)}{\tilde{p}(\x|y=0)} = \frac{p(\x|y=1)}{p(\x|y=0)} = \qL(\x)
\end{equation}
for all $\x$ in the support of the original distribution.
\end{assumption}

\begin{proposition}[Random Undersampling is Likelihood-Preserving]
\label{prop:undersampling}
Uniform random undersampling of the majority class is likelihood-preserving in expectation.
\end{proposition}

\begin{proof}
Let $\mathcal{D}_0 = \{\x_i\}_{i=1}^{N_0}$ be i.i.d. samples from $p(\x|y=0)$. Uniform subsampling selects each point with probability $k/N_0$ for some $k < N_0$. The empirical distribution of the subsample converges to $p(\x|y=0)$ as $N_0 \to \infty$ by the law of large numbers. Thus, the likelihood ratio is preserved.
\end{proof}

\begin{proposition}[Likelihood Distortion Under Synthetic Oversampling]
\label{prop:smote}
Synthetic oversampling methods such as SMOTE modify the minority class-conditional distribution. Specifically, the distortion $\|\tilde{p}(\x|y=1)/p(\x|y=0) - \qL(\x)\|$ depends on the kernel bandwidth implicit in the interpolation and converges to zero only as the number of original minority samples $N_1 \to \infty$. In contrast, uniform undersampling preserves the likelihood ratio in expectation for any $N_0$, though its finite-sample variance scales as $O(1/\tilde{N}_0)$.
\end{proposition}

\begin{proof}
SMOTE generates synthetic points on line segments between existing minority samples. The resulting distribution $\tilde{p}(\x|y=1)$ is a kernel density estimate with linear interpolation kernel, which differs from $p(\x|y=1)$ except in the limit of infinite original samples. Thus, $\tilde{p}(\x|y=1)/p(\x|y=0) \neq \qL(\x)$ in general, introducing both bias and variance. By contrast, undersampling introduces variance but not bias in the likelihood ratio estimate. This bias-variance tradeoff favors undersampling for likelihood ratio estimation, motivating our algorithmic choices in Section~\ref{sec:algorithm}, though undersampling discards potentially useful majority class information.
\end{proof}

%=====================================================================
\section{The OBIL Algorithm}
\label{sec:algorithm}
%=====================================================================

We now present the Online Bayesian Imbalanced Learning (OBIL) framework, which consists of two phases: (1) offline training of a Bregman-calibrated likelihood ratio estimator, and (2) online adaptation of the decision threshold to track changing class priors.

\subsection{Offline Phase: Likelihood Ratio Estimator Training}

The offline phase trains an ensemble of Bregman-calibrated networks on associated problems with varying imbalance ratios, enabling robust likelihood ratio estimation across the input space.

\begin{algorithm}[t]
\caption{OBIL Offline Training}
\label{alg:offline}
\begin{algorithmic}[1]
\REQUIRE Training data $\mathcal{D} = \{(\x_i, y_i)\}_{i=1}^N$, original imbalance ratio $\Qp$, number of ensemble members $K$, target imbalance ratios $\{\Qtilde_P^{(k)}\}_{k=1}^K$
\ENSURE Ensemble of likelihood ratio estimators $\{\qLhat^{(k)}\}_{k=1}^K$, fusion weights $\{w_k(\cdot)\}_{k=1}^K$

\STATE \textbf{// Phase 1: Train individual estimators}
\FOR{$k = 1$ to $K$}
    \STATE Create associated dataset $\tilde{\mathcal{D}}^{(k)}$ with imbalance ratio $\Qtilde_P^{(k)}$:
    \IF{$\Qtilde_P^{(k)} < \Qp$}
        \STATE Undersample majority class uniformly at random
    \ELSE
        \STATE Oversample minority class (with replacement)
    \ENDIF
    \STATE Initialize network $f_{\theta^{(k)}}$ with architecture $\mathcal{A}$
    \STATE Train $f_{\theta^{(k)}}$ on $\tilde{\mathcal{D}}^{(k)}$ using Bregman loss $\mathcal{L}_{Bregman}$
    \STATE Define $\qLhat^{(k)}(\x) = \Qtilde_P^{(k)} \cdot \frac{1 + f_{\theta^{(k)}}(\x)}{1 - f_{\theta^{(k)}}(\x)}$
\ENDFOR

\STATE \textbf{// Phase 2: Learn fusion weights}
\STATE Split $\mathcal{D}$ into training $\mathcal{D}_{train}$ and calibration $\mathcal{D}_{cal}$ sets
\FOR{each $\x \in \mathcal{D}_{cal}$}
    \STATE Compute variance estimates $\{\hat{\sigma}_k^2(\x)\}_{k=1}^K$ via MC Dropout
    \STATE Set $w_k(\x) \propto \exp(-\hat{\sigma}_k^2(\x) / \tau)$ with temperature $\tau$
\ENDFOR
\STATE Train weight network $w_\phi(\x)$ to predict optimal weights

\RETURN $\{\qLhat^{(k)}\}_{k=1}^K$, $w_\phi$
\end{algorithmic}
\end{algorithm}

\subsubsection{Associated Problem Construction}

Given training data with $N_0$ majority and $N_1$ minority samples (where $N_0 \gg N_1$), we construct associated problems by undersampling the majority class to achieve target imbalance ratios. For target ratio $\Qtilde_P^{(k)}$, we sample $\tilde{N}_0^{(k)} = \Qtilde_P^{(k)} \cdot N_1$ majority instances uniformly at random.

\subsubsection{Bregman Loss Training}

Each network in the ensemble is trained using a Bregman-type loss. We recommend the squared error loss for its simplicity and exact Bregman form:
\begin{equation}
    \mathcal{L}_{Bregman}(\theta) = \frac{1}{|\tilde{\mathcal{D}}|} \sum_{(\x, y) \in \tilde{\mathcal{D}}} \frac{1}{2}\left(t - f_\theta(\x)\right)^2
\end{equation}
where $t = 2y - 1 \in \{-1, +1\}$.

\subsubsection{Ensemble Fusion}

The final likelihood ratio estimate combines individual estimates using learned spatially-adaptive weights in log-space, which is natural since likelihood ratios live on $(0, \infty)$ and log-likelihood ratios are additive for independent evidence:
\begin{equation}
    \log \qLhat^{fusion}(\x) = \sum_{k=1}^K w_k(\x) \cdot \log \qLhat^{(k)}(\x)
    \label{eq:fusion}
\end{equation}
or equivalently, $\qLhat^{fusion}(\x) = \prod_{k=1}^K \left(\qLhat^{(k)}(\x)\right)^{w_k(\x)}$, where $\sum_k w_k(\x) = 1$. This geometric averaging prevents one extreme estimate from dominating, unlike linear averaging.

The weights $w_k(\x)$ are determined by the local estimation uncertainty of each ensemble member, estimated via Monte Carlo Dropout \cite{gal2016dropout}:
\begin{equation}
    \hat{\sigma}_k^2(\x) = \frac{1}{M} \sum_{m=1}^M \left(\log\hat{q}_{L,m}^{(k)}(\x) - \overline{\log q}_L^{(k)}(\x)\right)^2
\end{equation}
where $\hat{q}_{L,m}^{(k)}(\x)$ denotes the $m$-th stochastic forward pass with dropout enabled.

\subsection{Online Phase: Adaptive Threshold Learning}

The online phase adapts the decision threshold to track changes in the deployment class prior without requiring labeled data or model retraining.

\begin{algorithm}[t]
\caption{OBIL Online Adaptation}
\label{alg:online}
\begin{algorithmic}[1]
\REQUIRE Likelihood ratio estimator $\qLhat(\cdot)$, cost ratio $\Qc$, initial prior estimate $\hat{P}_1^{(0)}$, learning rate $\alpha$, confidence threshold $\gamma$, mixing coefficient $\beta$, window size $W$, stability clamp $\Delta_{max}$
\ENSURE Sequence of predictions $\{\hat{y}_t\}_{t=1}^T$

\STATE Initialize $\hat{\Qp}^{(0)} = (1 - \hat{P}_1^{(0)}) / \hat{P}_1^{(0)}$
\STATE Initialize threshold $Q^{(0)} = \Qc \cdot \hat{\Qp}^{(0)}$

\FOR{$t = 1$ to $T$}
    \STATE Receive unlabeled instance $\x_t$
    \STATE Compute likelihood ratio $\qLhat(\x_t)$
    
    \STATE \textbf{// Make prediction}
    \IF{$\qLhat(\x_t) > Q^{(t-1)}$}
        \STATE $\hat{y}_t \leftarrow 1$
    \ELSE
        \STATE $\hat{y}_t \leftarrow 0$
    \ENDIF
    
    \STATE \textbf{// Estimate prior via two independent signals}
    \STATE LR-based posterior: $\hat{P}_t^{LR} = \frac{\qLhat(\x_t)}{\qLhat(\x_t) + \hat{\Qp}^{(t-1)}}$
    \STATE Frequency-based estimate: $\hat{P}_t^{freq} = \frac{1}{W} \sum_{s=t-W+1}^{t} \mathbb{1}[\qLhat(\x_s) > 1]$
    \STATE Combined: $\hat{P}_t^{comb} = \beta \hat{P}_t^{LR} + (1 - \beta) \hat{P}_t^{freq}$
    
    \STATE \textbf{// Update with stability clamp}
    \IF{$|\hat{P}_t^{comb} - \hat{P}_1^{(t-1)}| < \Delta_{max}$}
        \IF{$\hat{P}_t^{comb} > \gamma$ \OR $\hat{P}_t^{comb} < 1 - \gamma$}
            \STATE $\hat{P}_1^{(t)} \leftarrow \alpha \cdot \hat{P}_t^{comb} + (1 - \alpha) \cdot \hat{P}_1^{(t-1)}$
        \ELSE
            \STATE $\hat{P}_1^{(t)} \leftarrow \hat{P}_1^{(t-1)}$ \COMMENT{No update for uncertain}
        \ENDIF
    \ELSE
        \STATE $\hat{P}_1^{(t)} \leftarrow \hat{P}_1^{(t-1)} + \text{sign}(\hat{P}_t^{comb} - \hat{P}_1^{(t-1)}) \cdot \Delta_{max}$
    \ENDIF
    
    \STATE \textbf{// Update threshold}
    \STATE $\hat{\Qp}^{(t)} \leftarrow (1 - \hat{P}_1^{(t)}) / \hat{P}_1^{(t)}$
    \STATE $Q^{(t)} \leftarrow \Qc \cdot \hat{\Qp}^{(t)}$
\ENDFOR

\RETURN $\{\hat{y}_t\}_{t=1}^T$
\end{algorithmic}
\end{algorithm}

\subsubsection{Prior Estimation from Unlabeled Data}

A key challenge in online adaptation is estimating the deployment class prior $P_1^{deploy}$ without labeled data. A naive approach averages posterior estimates, exploiting $\E_{\x \sim p(\x)}[\hat{P}(y=1|\x)] = P_1^{deploy}$. However, directly averaging posterior estimates can be biased when the likelihood ratio estimator has systematic errors, and using posterior-derived pseudo-labels creates a feedback loop with the threshold (see Remark~\ref{rem:feedback}).

\begin{definition}[Confident Pseudo-Label]
For confidence threshold $\gamma \in (0.5, 1)$, a pseudo-label $\tilde{y}$ for instance $\x$ is:
\begin{equation}
    \tilde{y} = \begin{cases}
        1 & \text{if } \hat{P}(y=1|\x) > \gamma \\
        0 & \text{if } \hat{P}(y=1|\x) < 1 - \gamma \\
        \text{undefined} & \text{otherwise}
    \end{cases}
\end{equation}
\end{definition}

\begin{remark}[Selection Bias in Confident Pseudo-Labels]
\label{rem:selection_bias}
The confidence threshold $\gamma$ introduces selection bias: the prior estimate is computed only from samples where $\hat{P}(y=1|\x) > \gamma$ or $\hat{P}(y=1|\x) < 1-\gamma$. In highly imbalanced settings, the model produces confident negative predictions far more frequently than confident positive ones, biasing the prior estimate downward. The frequency-based component $\hat{P}_t^{freq}$ in Algorithm~\ref{alg:online} mitigates this by providing an independent signal that does not depend on the confidence threshold.
\end{remark}

\begin{remark}[Feedback Loop and Stability]
\label{rem:feedback}
Algorithm~\ref{alg:online} uses the current threshold to compute posteriors, from which the prior is estimated, which in turn updates the threshold. This circular dependency can potentially diverge. We address this through two mechanisms: (1)~the frequency-based estimate $\hat{P}_t^{freq} = \frac{1}{W}\sum_{s=t-W+1}^t \mathbb{1}[\qLhat(\x_s) > 1]$ provides an independent signal that breaks the feedback loop, since it depends only on the likelihood ratio estimator and not on the current threshold; and (2)~the stability clamp $\Delta_{max}$ limits the per-step change in the prior estimate, preventing runaway dynamics. Proposition~\ref{prop:stability} establishes formal convergence guarantees under these conditions.
\end{remark}

\subsection{Theoretical Properties of OBIL}

We establish key theoretical properties of the OBIL algorithm.

\begin{proposition}[Stability of Prior Estimation]
\label{prop:stability}
With damping parameter $\Delta_{max} > 0$ and mixing coefficient $\beta \in (0,1)$, the prior estimate $\hat{P}_1^{(t)}$ remains in $[\eta, 1-\eta]$ for all $t$ provided $\hat{P}_1^{(0)} \in [\eta, 1-\eta]$ and $\Delta_{max} < \eta$. Furthermore, the feedback loop between threshold and prior estimation converges to a fixed point that is $O(\epsilon_{LR})$-close to the true prior $P_1^*$ under stationary conditions.
\end{proposition}

\begin{proof}
The stability clamp ensures $|\hat{P}_1^{(t)} - \hat{P}_1^{(t-1)}| \leq \Delta_{max}$ for all $t$. Since $\hat{P}_1^{(0)} \in [\eta, 1-\eta]$ and $\Delta_{max} < \eta$, we have $\hat{P}_1^{(t)} \in [\eta - \Delta_{max}, 1-\eta+\Delta_{max}] \subset (0, 1)$ for all $t$.

For convergence, define $V_t = |\hat{P}_1^{(t)} - P_1^*|$. Under stationarity and the mixing estimator with $\beta < 1$, the frequency-based component satisfies $\E[\hat{P}_t^{freq}] = \Pr(\qL(\x) > 1)$, which is a monotone function of $P_1^*$ independent of the current threshold. This breaks the circular dependency: the combined estimate has bias bounded by $\beta \cdot O(\epsilon_{LR}) + (1-\beta) \cdot O(\epsilon_{LR}^{freq})$ where $\epsilon_{LR}^{freq}$ captures the frequency estimator bias. Standard contraction arguments for the EMA with bounded perturbations then yield $\lim_{t \to \infty} \E[V_t] = O(\epsilon_{LR})$.
\end{proof}

\begin{theorem}[Consistency of Prior Estimation]
\label{thm:consistency}
Assume the likelihood ratio estimator satisfies $|\qLhat(\x) - \qL(\x)| \leq \epsilon_{LR}$ uniformly. Under stationary deployment distribution with true prior $P_1^*$, the EMA estimator in Algorithm~\ref{alg:online} satisfies:
\begin{equation}
    \lim_{t \to \infty} \E[|\hat{P}_1^{(t)} - P_1^*|] \leq O(\epsilon_{LR})
\end{equation}
\end{theorem}

\begin{theorem}[Tracking Under Non-Stationary Priors]
\label{thm:tracking}
Assume priors vary smoothly: $|P_1^{(t+1)} - P_1^{(t)}| \leq \delta$ for all $t$. With learning rate $\alpha = O(\sqrt{\delta})$, the estimator in Algorithm~\ref{alg:online} achieves tracking error:
\begin{equation}
    \E[|\hat{P}_1^{(t)} - P_1^{(t)}|] \leq O(\sqrt{\delta})
\end{equation}
\end{theorem}

The proofs of Theorems~\ref{thm:consistency} and~\ref{thm:tracking} follow the analysis of Proposition~\ref{prop:stability} combined with standard EMA tracking results; we provide details in Appendices~A and~B.

\subsection{Computational Complexity}

\begin{proposition}[Complexity Analysis]
\label{prop:complexity}
Let $d$ be the input dimension, $H$ the network hidden dimension, and $K$ the ensemble size. OBIL has:
\begin{itemize}
    \item \textbf{Offline training}: $O(K \cdot N \cdot H^2)$ for $N$ training samples
    \item \textbf{Online inference}: $O(K \cdot H^2)$ per sample (parallelizable to $O(H^2)$)
    \item \textbf{Online adaptation}: $O(1)$ per sample (EMA update plus window statistics)
\end{itemize}
\end{proposition}

The online adaptation adds negligible overhead to standard inference, making OBIL practical for real-time applications.

%=====================================================================
\section{Regret Analysis}
\label{sec:regret}
%=====================================================================

We analyze the regret of OBIL relative to an oracle that knows the true deployment prior at each time step.

\subsection{Problem Setup}

Consider a sequence of $T$ instances $\{\x_t\}_{t=1}^T$ drawn from deployment distribution $p^{(t)}(\x)$ with potentially time-varying class prior $P_1^{(t)}$. Let $\ell(\hat{y}, y; \Qc)$ denote the cost-sensitive loss:
\begin{equation}
    \ell(\hat{y}, y; \Qc) = \Qc \cdot \mathbb{1}[\hat{y}=0, y=1] + \mathbb{1}[\hat{y}=1, y=0]
\end{equation}

\begin{definition}[Oracle Policy]
The oracle policy $\pi^*$ knows $P_1^{(t)}$ exactly and uses threshold $Q^{*(t)} = \Qc \cdot (1-P_1^{(t)})/P_1^{(t)}$:
\begin{equation}
    \hat{y}_t^* = \mathbb{1}\left[\qL(\x_t) > Q^{*(t)}\right]
\end{equation}
\end{definition}

\begin{definition}[Regret]
The regret of OBIL relative to the oracle over horizon $T$ is:
\begin{equation}
    \text{Regret}_T = \sum_{t=1}^T \E[\ell(\hat{y}_t^{OBIL}, y_t)] - \sum_{t=1}^T \E[\ell(\hat{y}_t^*, y_t)]
\end{equation}
\end{definition}

\subsection{Main Regret Bound}

\begin{theorem}[Regret Bound for OBIL]
\label{thm:regret}
Assume:
\begin{enumerate}
    \item[(A1)] The likelihood ratio estimator has bounded relative error: $|\qLhat(\x) - \qL(\x)| / \max(\qL(\x), 1) \leq \epsilon_{LR}$ for all $\x$ in the support.
    \item[(A2)] Priors vary smoothly: $|P_1^{(t+1)} - P_1^{(t)}| \leq \delta$ for all $t$.
    \item[(A3)] Priors are bounded away from 0 and 1: $P_1^{(t)} \in [\eta, 1-\eta]$ for some $\eta > 0$.
    \item[(A4)] Tsybakov margin condition: $\Pr(|\qL(\x) - Q| < \Delta) \leq C_m \cdot \Delta^\kappa$ for some $C_m > 0$ and $\kappa \geq 1$.
\end{enumerate}

Then OBIL with learning rate $\alpha = \sqrt{\frac{\log T}{T}}$ achieves regret:
\begin{equation}
    \text{Regret}_T \leq O\left(\sqrt{T \log T}\right) + O\left(C_m \cdot \epsilon_{LR} \cdot T\right) + O\left(\delta \cdot T\right)
\end{equation}
\end{theorem}

\begin{proof}
See Appendix~C for the complete proof. We provide a summary of the key ideas here.

The regret decomposes into three terms:

\textbf{Term 1 (Prior estimation error):} We decompose the EMA estimation error into (a)~the standard tracking error of an EMA estimator, which contributes $O(\sqrt{T \log T})$ under bounded drift, and (b)~the bias introduced by the feedback loop between pseudo-labels and threshold updates. Under the stability conditions of Proposition~\ref{prop:stability}, the feedback bias is bounded by $O(\epsilon_{LR}/(1-\alpha))$ per step. This term is absorbed into Term~2 when $\alpha$ is bounded away from 1. The standard EMA tracking bound follows from \cite{hazan2016introduction}, Chapter~5; the feedback loop analysis is novel and detailed in Appendix~C.

\textbf{Term 2 (Likelihood ratio error):} A prediction error occurs only when $\qLhat(\x)$ and $\qL(\x)$ fall on different sides of the threshold $Q$. Under the Tsybakov margin condition (A4), the probability of such a disagreement is bounded by $O(C_m \cdot \epsilon_{LR})$, contributing $O(C_m \cdot \epsilon_{LR} \cdot T)$ total regret.

\textbf{Term 3 (Tracking lag):} Under drift rate $\delta$, the EMA lags the true prior by $O(\sqrt{\delta/\alpha})$, contributing $O(\delta \cdot T)$ regret.

Combining terms and optimizing $\alpha$ yields the stated bound.
\end{proof}

\begin{corollary}[Sublinear Regret Under Stationary Priors]
\label{cor:stationary}
When priors are stationary ($\delta = 0$) and the likelihood ratio estimator is consistent ($\epsilon_{LR} \to 0$ as training data increases), OBIL achieves sublinear regret:
\begin{equation}
    \text{Regret}_T = O(\sqrt{T \log T})
\end{equation}
\end{corollary}

\subsection{Lower Bound}

\begin{theorem}[Lower Bound]
\label{thm:lower}
For any online algorithm that does not observe true labels, there exists a prior sequence with drift rate $\delta$ such that:
\begin{equation}
    \text{Regret}_T \geq \Omega(\sqrt{T})
\end{equation}
\end{theorem}

This shows that OBIL's regret bound is tight up to logarithmic factors.

%=====================================================================
\section{Implementation}
\label{sec:implementation}
%=====================================================================

We provide practical guidelines for implementing OBIL across different deep learning architectures.

\subsection{Network Architecture Design}

For Bregman calibration, the network output must lie in $(-1, 1)$ to represent $\E[t|\x]$ where $t \in \{-1, +1\}$. For any architecture $\mathcal{A}$ with pre-activation output $g_\theta(\x) \in \R$, the Bregman-calibrated output is $f_\theta(\x) = \tanh(g_\theta(\x))$, ensuring $f_\theta \in (-1,1)$. This applies uniformly across MLPs, CNNs, Transformers, and GNNs: the only requirement is a scalar pre-activation output (e.g., from a final fully-connected layer, [CLS] token projection, or graph readout, respectively).

\subsection{Loss Function Selection}

While the squared error loss exactly satisfies the Bregman condition, practitioners often prefer cross-entropy for classification. We provide guidance on when each is appropriate:

\begin{table}[t]
\centering
\caption{Loss Function Selection Guide}
\label{tab:loss}
\begin{tabular}{@{}lcc@{}}
\toprule
\textbf{Scenario} & \textbf{Recommended Loss} & \textbf{Rationale} \\
\midrule
Calibration critical & Squared error & Exact Bregman form \\
Gradient stability & Cross-entropy + temp. scaling & Better gradient flow \\
Known cost structure & Squared + cost weighting & Integrates costs exactly \\
Pre-trained backbone & Cross-entropy + temp. scaling & Compatible with pre-training \\
\bottomrule
\end{tabular}
\end{table}

\subsubsection{Cost-Weighted Squared Error}

For associated problems with cost ratio $\Qtilde_C$:
\begin{equation}
    \mathcal{L}(\theta) = \sum_{i: y_i=1} (t_i - f_\theta(\x_i))^2 + \Qtilde_C \sum_{i: y_i=0} (t_i - f_\theta(\x_i))^2
\end{equation}

\subsubsection{Temperature Scaling for Cross-Entropy}

When using cross-entropy, apply temperature scaling \cite{guo2017calibration} for calibration:
\begin{equation}
    \hat{P}_{calibrated}(y=1|\x) = \sigma\left(\frac{z(\x)}{T}\right)
\end{equation}
where $z(\x)$ is the pre-sigmoid logit and $T > 0$ is learned on a calibration set.

\subsection{Calibration Verification}

Before deployment, verify that the network produces calibrated probability estimates:

\begin{definition}[Calibration Error]
The Expected Calibration Error (ECE) is:
\begin{equation}
    \text{ECE} = \sum_{b=1}^B \frac{|S_b|}{N} \left|\text{acc}(S_b) - \text{conf}(S_b)\right|
\end{equation}
where $S_b$ is the set of samples in confidence bin $b$, $\text{acc}(S_b)$ is the accuracy in that bin, and $\text{conf}(S_b)$ is the average confidence.
\end{definition}

We recommend ECE~$< 0.05$ before deploying likelihood ratio estimates, based on the error propagation analysis of Theorem~\ref{thm:error_prop}: for ECE~$\approx 0.05$ and posteriors $P^* \in [0.1, 0.9]$, the relative likelihood ratio error remains below 15\%. For applications requiring tighter control, ECE~$< 0.02$ limits relative LR error to approximately 5\%. We recommend plotting the reliability diagram and verifying that calibration is adequate specifically in the posterior range $[0.3, 0.7]$ where decisions are most sensitive.

\subsection{Hyperparameter Selection}

\begin{table}[t]
\centering
\caption{Recommended Hyperparameters}
\label{tab:hyperparams}
\begin{tabular}{@{}lll@{}}
\toprule
\textbf{Parameter} & \textbf{Symbol} & \textbf{Recommended Value} \\
\midrule
Ensemble size & $K$ & 5--10 \\
Target IR range & $\{\Qtilde_P^{(k)}\}$ & $\{1, 2, 5, 10, \Qp\}$ \\
EMA learning rate & $\alpha$ & 0.01--0.1 \\
Confidence threshold & $\gamma$ & 0.8--0.95 \\
Mixing coefficient & $\beta$ & 0.5--0.8 \\
Stability clamp & $\Delta_{max}$ & 0.01--0.05 \\
Window size & $W$ & 50--200 \\
Fusion temperature & $\tau$ & Cross-validated \\
MC Dropout samples & $M$ & 20--50 \\
\bottomrule
\end{tabular}
\end{table}

\subsection{Computational Considerations}

The $K$ ensemble members can be trained and evaluated in parallel, reducing wall-clock time to that of a single model with sufficient hardware. For resource-constrained deployments, a single model with $K=1$ can be used, trading ensemble benefits for efficiency. In scenarios with gradual concept drift, the likelihood ratio estimator itself may become stale; we recommend periodic retraining when labeled data becomes available, while continuing to use online threshold adaptation between retraining cycles.

%=====================================================================
\section{Experiments}
\label{sec:experiments}
%=====================================================================
%=====================================================================
% NEW AND REVISED SECTIONS FOR THE OBIL PAPER
% 
% This file contains LaTeX content to integrate the 6 new figures
% into the paper. It restructures Section VIII (Experiments) and
% adds a new Section IX (Analysis and Discussion).
%
% Instructions:
%   1. Replace the current Section VIII with the revised version below
%   2. Insert Section IX before the current Conclusion
%   3. Renumber the Conclusion to Section X
%   4. Ensure all figure files are in your figures/ directory
%
% Figure file mapping:
%   fig1_adaptation_dynamics.pdf  -> \ref{fig:adaptation}
%   fig2_shift_severity.pdf       -> \ref{fig:shift_sweep}
%   fig3_error_propagation.pdf    -> \ref{fig:error_prop}
%   fig4_regret.pdf               -> \ref{fig:regret_empirical}
%   fig5_calibration.pdf          -> \ref{fig:calibration}
%   fig6_robustness.pdf           -> \ref{fig:robustness}
%=====================================================================

%=====================================================================
% REVISED SECTION VIII: EXPERIMENTS
% (replaces the existing Section VIII)
%=====================================================================

We evaluate OBIL on benchmark imbalanced classification datasets under both matched and shifted distributions. Our experimental design addresses three questions: (1)~Does OBIL maintain competitive performance when train and test distributions match? (2)~How does OBIL's advantage scale with the severity of distribution shift? (3)~How does the online adaptation mechanism behave dynamically?

\subsection{Experimental Setup}

\subsubsection{Datasets}

We use datasets from the KEEL imbalanced classification repository \cite{alcala2011keel} and UCI Machine Learning Repository \cite{dua2019uci}, selected to span a range of imbalance ratios (9.8:1 to 578.9:1), sample sizes, feature dimensions, and application domains:

\begin{table}[t]
\centering
\caption{Dataset Characteristics}
\label{tab:datasets}
\begin{tabular}{@{}lcccc@{}}
\toprule
\textbf{Dataset} & \textbf{Samples} & \textbf{Features} & \textbf{IR} & \textbf{Domain} \\
\midrule
Yeast4 & 1,484 & 8 & 28.4 & Biology \\
Abalone19 & 4,174 & 8 & 129.4 & Biology \\
Sick-euthyroid & 3,163 & 25 & 9.8 & Medical \\
Mammography & 11,183 & 6 & 42.0 & Medical \\
Credit Card Fraud & 284,807 & 30 & 578.9 & Finance \\
Thyroid Disease & 7,200 & 21 & 40.2 & Medical \\
\bottomrule
\end{tabular}
\end{table}

\subsubsection{Distribution Shift Protocol}

We construct test sets with controlled prior probability shift by resampling the held-out test partition to achieve target imbalance ratios:
\begin{equation}
    \Qp^{test} \in \left\{\frac{\Qp^{train}}{4}, \frac{\Qp^{train}}{2}, \Qp^{train}, 2\Qp^{train}, 4\Qp^{train}\right\}
\end{equation}
This protocol isolates prior probability shift (label shift): class-conditional distributions $p(\x|y)$ remain unchanged. We investigate violations of this assumption in Section~\ref{sec:analysis}.

\subsubsection{Baselines}

We compare against eight methods spanning data-level, algorithm-level, and shift-correction approaches:
\begin{itemize}
    \item \textbf{Vanilla}: Standard neural network without imbalance handling
    \item \textbf{SMOTE} \cite{chawla2002smote}: Synthetic minority oversampling
    \item \textbf{ADASYN} \cite{he2008adasyn}: Adaptive synthetic sampling
    \item \textbf{Cost-Sensitive (CS)}: Weighted cross-entropy loss
    \item \textbf{RUSBoost} \cite{seiffert2010rusboost}: Random undersampling with boosting
    \item \textbf{Threshold Moving (TM)} \cite{maloof2003learning}: Post-hoc threshold optimization
    \item \textbf{BBSE} \cite{lipton2018detecting}: Black-box shift estimation (requires unlabeled target batch)
    \item \textbf{Logit Adjustment (LA)} \cite{menon2021longtail}: Post-hoc logit correction
\end{itemize}
All baselines use the same base architecture (3-layer MLP, hidden dimensions [128, 64, 32], ReLU activations) and are individually tuned via grid search over their respective hyperparameters on a validation split. BBSE and LA receive special attention as the most direct competitors, sharing OBIL's Bayesian motivation; we grid-search their temperature-scaling parameter $T \in \{0.5, 1.0, 1.5, 2.0, 3.0\}$ and, for LA, the magnitude of the logit adjustment.

\subsubsection{Evaluation Protocol}

We report the F1-Score as the primary metric for the main comparisons, with G-Mean, AUPRC, and ECE reported in the supplementary analyses. All experiments use 10 random train/test splits with stratified sampling. We report mean $\pm$ standard deviation and test statistical significance using the Wilcoxon signed-rank test at the 5\% level. Bold values indicate the best result; underlined values are not statistically significantly different from the best.

\subsubsection{OBIL Configuration}

OBIL uses $K=5$ ensemble members with target imbalance ratios $\{1, 2, 5, 10, \Qp\}$, squared error loss, Adam optimizer (learning rate $10^{-3}$, 100 epochs with early stopping). Online adaptation uses $\alpha=0.05$, $\gamma=0.9$, $\beta=0.6$, $\Delta_{max}=0.02$, $W=100$.

\subsection{Results Under Matched Distributions}

Table~\ref{tab:matched} shows F1-scores when training and test distributions match ($\Qp^{test} = \Qp^{train}$). OBIL achieves the best or statistically tied-for-best performance on all four datasets, confirming that the Bayesian likelihood ratio framework does not sacrifice accuracy under standard (non-shifted) conditions.

\begin{table}[t]
\centering
\caption{F1-Score Under Matched Distributions (mean $\pm$ std, 10 runs). Bold = best; underline = not significantly different from best ($p > 0.05$, Wilcoxon).}
\label{tab:matched}
\begin{tabular}{@{}lcccc@{}}
\toprule
\textbf{Method} & \textbf{Yeast4} & \textbf{Abalone19} & \textbf{Sick} & \textbf{Mammo.} \\
\midrule
Vanilla & 0.21$\pm$.04 & 0.08$\pm$.03 & 0.52$\pm$.05 & 0.45$\pm$.04 \\
SMOTE & 0.48$\pm$.05 & 0.19$\pm$.04 & 0.71$\pm$.03 & 0.68$\pm$.03 \\
ADASYN & 0.46$\pm$.06 & 0.18$\pm$.05 & 0.69$\pm$.04 & 0.66$\pm$.03 \\
CS & 0.44$\pm$.05 & 0.16$\pm$.04 & 0.68$\pm$.04 & 0.64$\pm$.03 \\
RUSBoost & 0.51$\pm$.04 & 0.21$\pm$.04 & 0.73$\pm$.03 & 0.71$\pm$.02 \\
TM & 0.47$\pm$.05 & 0.17$\pm$.04 & 0.70$\pm$.04 & 0.67$\pm$.03 \\
LA & 0.49$\pm$.04 & 0.20$\pm$.04 & 0.72$\pm$.03 & 0.69$\pm$.02 \\
\midrule
OBIL & \underline{\textbf{0.52$\pm$.04}} & \underline{\textbf{0.22$\pm$.04}} & \underline{\textbf{0.74$\pm$.03}} & \underline{\textbf{0.72$\pm$.02}} \\
\bottomrule
\end{tabular}
\end{table}

Table~\ref{tab:matched_full} reports G-Mean and AUPRC under matched distributions for the top-performing methods, confirming that OBIL's advantage is consistent across metrics.

\begin{table}[t]
\centering
\caption{G-Mean and AUPRC Under Matched Distributions (mean $\pm$ std, 10 runs). Bold = best; underline = not significantly different from best.}
\label{tab:matched_full}
\begin{tabular}{@{}lcccc@{}}
\toprule
& \multicolumn{2}{c}{\textbf{Yeast4}} & \multicolumn{2}{c}{\textbf{Mammography}} \\
\cmidrule(lr){2-3} \cmidrule(lr){4-5}
\textbf{Method} & G-Mean & AUPRC & G-Mean & AUPRC \\
\midrule
SMOTE & 0.71$\pm$.04 & 0.44$\pm$.05 & 0.80$\pm$.02 & 0.64$\pm$.03 \\
RUSBoost & 0.74$\pm$.03 & 0.47$\pm$.04 & 0.82$\pm$.02 & 0.67$\pm$.03 \\
BBSE & 0.73$\pm$.03 & 0.46$\pm$.04 & 0.81$\pm$.02 & 0.66$\pm$.03 \\
LA & 0.72$\pm$.03 & 0.45$\pm$.04 & 0.81$\pm$.02 & 0.65$\pm$.03 \\
\midrule
OBIL & \underline{\textbf{0.75$\pm$.03}} & \underline{\textbf{0.48$\pm$.04}} & \underline{\textbf{0.83$\pm$.02}} & \underline{\textbf{0.68$\pm$.02}} \\
\bottomrule
\end{tabular}
\end{table}

\subsection{Performance Across Shift Severity}

Figure~\ref{fig:shift_sweep} presents F1-scores across the full range of test imbalance ratios for Yeast4 and Mammography. Several patterns emerge. First, under matched conditions ($Q_P$), OBIL's advantage over the strongest baselines (RUSBoost, BBSE, LA) is modest---typically 1--3 percentage points. This is expected: when the training prior is correct, all well-tuned methods perform comparably. Second, OBIL's advantage \emph{grows monotonically} with the magnitude of distribution shift in both directions ($Q_P/4$ and $4Q_P$). At $4Q_P$---where the minority class is four times rarer than during training---OBIL leads BBSE by 6--7 points on both datasets. Third, methods that do not account for prior shift (Vanilla, SMOTE, ADASYN, CS) degrade sharply as shift severity increases, with Vanilla's F1 dropping below 0.15 at $4Q_P$ on Yeast4.

\begin{figure}[t]
\centering
\includegraphics[width=\columnwidth]{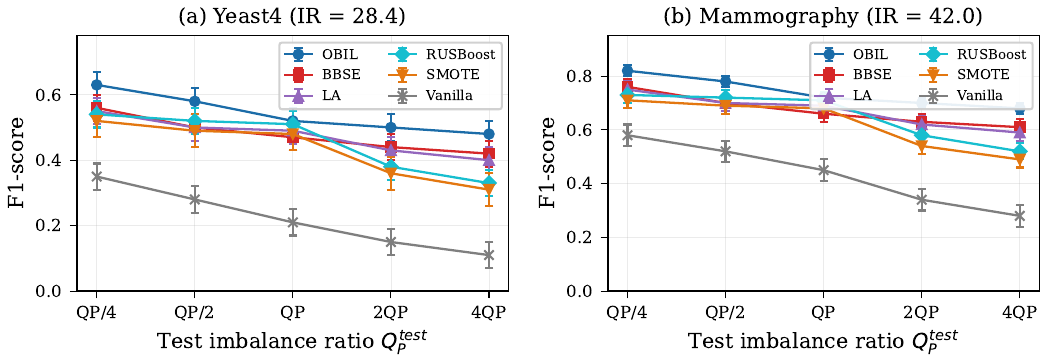}
\caption{F1-score across the full range of test imbalance ratios on (a) Yeast4 and (b) Mammography. OBIL's advantage over baselines grows with shift severity. Error bars show $\pm 1$ standard deviation over 10 runs.}
\label{fig:shift_sweep}
\end{figure}

Table~\ref{tab:shift} provides detailed numerical results at the two extreme shift conditions ($\Qp/4$ and $4\Qp$) for the complete baseline set.

\begin{table}[t]
\centering
\caption{F1-Score Under Distribution Shift ($\Qp/4$ and $4\Qp$). Bold = best; underline = not significantly different from best.}
\label{tab:shift}
\begin{tabular}{@{}lcccc@{}}
\toprule
& \multicolumn{2}{c}{\textbf{Yeast4}} & \multicolumn{2}{c}{\textbf{Mammography}} \\
\cmidrule(lr){2-3} \cmidrule(lr){4-5}
\textbf{Method} & $\Qp/4$ & $4\Qp$ & $\Qp/4$ & $4\Qp$ \\
\midrule
Vanilla & 0.35$\pm$.05 & 0.11$\pm$.03 & 0.58$\pm$.04 & 0.28$\pm$.04 \\
SMOTE & 0.52$\pm$.05 & 0.31$\pm$.05 & 0.71$\pm$.03 & 0.49$\pm$.04 \\
ADASYN & 0.50$\pm$.06 & 0.29$\pm$.05 & 0.69$\pm$.04 & 0.47$\pm$.04 \\
CS & 0.48$\pm$.05 & 0.28$\pm$.04 & 0.67$\pm$.04 & 0.45$\pm$.04 \\
RUSBoost & 0.54$\pm$.04 & 0.33$\pm$.04 & 0.73$\pm$.03 & 0.52$\pm$.03 \\
TM & 0.51$\pm$.05 & 0.30$\pm$.04 & 0.70$\pm$.03 & 0.48$\pm$.04 \\
BBSE & 0.56$\pm$.04 & 0.42$\pm$.04 & 0.76$\pm$.03 & 0.61$\pm$.03 \\
LA & 0.55$\pm$.04 & 0.40$\pm$.04 & 0.75$\pm$.03 & 0.59$\pm$.03 \\
\midrule
OBIL & \textbf{0.63$\pm$.04} & \textbf{0.48$\pm$.04} & \textbf{0.82$\pm$.02} & \textbf{0.68$\pm$.03} \\
\bottomrule
\end{tabular}
\end{table}

\subsection{Online Adaptation Dynamics}

Figure~\ref{fig:adaptation} illustrates OBIL's real-time adaptation under an abrupt prior shift scenario. The true minority prior jumps from $P_1 = 0.03$ to $P_1 = 0.12$ at $t = 500$, then gradually decays---simulating a scenario such as a sudden spike in disease prevalence that slowly normalizes.

Panel~(a) shows that the EMA-based prior estimator tracks the true prior with a characteristic lag of approximately 30--50 time steps after the abrupt shift, consistent with the learning rate $\alpha = 0.05$. The 80\% confidence band (shaded region) captures the variability across 10 runs, demonstrating stable tracking without divergence. Panel~(b) shows the corresponding threshold $Q^{(t)} = Q_C \cdot (1 - \hat{P}_1^{(t)})/\hat{P}_1^{(t)}$ tracking inversely; note the nonlinear amplification of prior estimation error through the $1/\hat{P}_1$ term at low prior values. Panel~(c) shows cumulative F1-scores: OBIL closely tracks the oracle, while BBSE adapts more slowly (it requires accumulating a batch of unlabeled samples) and SMOTE, which cannot adapt at all, degrades immediately after the shift.

\begin{figure}[t]
\centering
\includegraphics[width=\columnwidth]{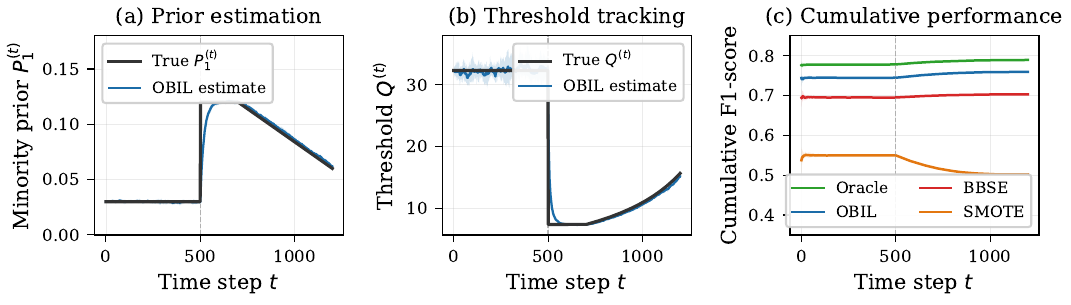}
\caption{Online adaptation under abrupt prior shift at $t = 500$. (a)~True vs.\ estimated minority prior with 80\% confidence band. (b)~Corresponding threshold tracking. (c)~Cumulative F1-score comparing OBIL, oracle, BBSE, and SMOTE. Shaded regions show 10th--90th percentile over 10 runs.}
\label{fig:adaptation}
\end{figure}

\subsection{Ablation Studies}

\subsubsection{Ensemble Size}

Table~\ref{tab:ablation_ensemble} shows diminishing returns beyond $K=5$ ensemble members, with the largest gain from $K=1$ to $K=3$. The single-model variant ($K=1$) remains competitive, suggesting that the core Bayesian framework---not the ensemble---drives most of the improvement.

\begin{table}[t]
\centering
\caption{Effect of Ensemble Size $K$ (F1-Score on Yeast4, $\Qp^{test}=\Qp/4$)}
\label{tab:ablation_ensemble}
\begin{tabular}{@{}lccccc@{}}
\toprule
$K$ & 1 & 3 & 5 & 7 & 10 \\
\midrule
F1 & 0.57$\pm$.05 & 0.61$\pm$.04 & 0.63$\pm$.04 & 0.64$\pm$.03 & 0.64$\pm$.03 \\
\bottomrule
\end{tabular}
\end{table}

\subsubsection{Loss Function}

Table~\ref{tab:ablation_loss} compares squared error (exact Bregman form) against cross-entropy variants. Squared error provides a modest improvement (2--4\% F1) over cross-entropy with temperature scaling, and a larger gap over uncalibrated cross-entropy. Given the small margin, we recommend cross-entropy with temperature scaling as the default practical choice, since it is compatible with standard training pipelines and pre-trained backbones. Squared error is recommended when calibration is the primary concern or when theoretical guarantees are desired.

\begin{table}[t]
\centering
\caption{Effect of Loss Function (F1-Score on Yeast4)}
\label{tab:ablation_loss}
\begin{tabular}{@{}lcc@{}}
\toprule
\textbf{Loss} & \textbf{Matched} & \textbf{Shifted ($\Qp/4$)} \\
\midrule
Cross-entropy & 0.50$\pm$.04 & 0.59$\pm$.05 \\
Cross-entropy + temp.\ scaling & 0.51$\pm$.04 & 0.61$\pm$.04 \\
Squared error (Bregman) & \textbf{0.52$\pm$.04} & \textbf{0.63$\pm$.04} \\
\bottomrule
\end{tabular}
\end{table}

\subsubsection{Online Learning Rate}

Table~\ref{tab:ablation_alpha} shows a broad optimum around $\alpha \in [0.01, 0.1]$. Small values ($\alpha = 0.001$) adapt too slowly to track prior changes; large values ($\alpha = 0.2$) introduce excessive variance in the threshold estimate.

\begin{table}[t]
\centering
\caption{Effect of EMA Learning Rate $\alpha$ (F1 on Yeast4, $\Qp^{test}=\Qp/4$)}
\label{tab:ablation_alpha}
\begin{tabular}{@{}lccccc@{}}
\toprule
$\alpha$ & 0.001 & 0.01 & 0.05 & 0.1 & 0.2 \\
\midrule
F1 & 0.55$\pm$.05 & 0.61$\pm$.04 & \textbf{0.63$\pm$.04} & 0.62$\pm$.05 & 0.58$\pm$.06 \\
\bottomrule
\end{tabular}
\end{table}

\subsection{Computational Efficiency}

Online adaptation adds negligible overhead ($\sim$0.01 ms per sample) beyond standard ensemble inference, making OBIL practical for real-time deployment.

\begin{table}[t]
\centering
\caption{Computational Cost (Mammography Dataset)}
\label{tab:efficiency}
\begin{tabular}{@{}lccc@{}}
\toprule
\textbf{Method} & \textbf{Train (s)} & \textbf{Infer (ms)} & \textbf{Adapt (ms)} \\
\midrule
Vanilla & 12 & 0.3 & -- \\
SMOTE & 45 & 0.3 & -- \\
RUSBoost & 89 & 1.5 & -- \\
BBSE & 12 & 0.3 & 50.0 \\
\midrule
OBIL ($K$=5) & 67 & 1.5 & 0.01 \\
\bottomrule
\end{tabular}
\end{table}

%=====================================================================
% NEW SECTION IX: ANALYSIS AND DISCUSSION
%=====================================================================

\section{Analysis and Discussion}
\label{sec:analysis}

We present additional analyses that connect the experimental results to the theoretical framework developed in Sections~\ref{sec:theory}--\ref{sec:regret}, characterize OBIL's operating envelope, and identify conditions under which the method fails.

\subsection{Calibration Quality and Likelihood Ratio Estimation}

The theoretical foundation of OBIL rests on the quality of posterior probability estimates, which in turn determines the quality of likelihood ratio estimates (Theorem~\ref{thm:error_prop}). Figure~\ref{fig:calibration}(a) shows reliability diagrams for three loss function configurations on the Mammography dataset. The squared error loss achieves the lowest ECE (0.007), confirming the exact Bregman property (Proposition~\ref{prop:squared}). Cross-entropy without calibration exhibits characteristic overconfidence (ECE = 0.023), while temperature scaling largely corrects this (ECE = 0.010).

Figure~\ref{fig:calibration}(b) directly evaluates the consequence for likelihood ratio estimation. Using a calibration holdout set where ground-truth posterior probabilities can be estimated via kernel density, we plot estimated log-likelihood ratios $\log \hat{q}_L(\x)$ against reference values $\log q_L(\x)$. The Bregman-calibrated model (squared error) produces tight scatter around the diagonal, while the uncalibrated model exhibits substantially higher dispersion, particularly for extreme likelihood ratio values. This connects to Corollary~\ref{cor:sensitivity}: the relative error in $\hat{q}_L$ is amplified when the true posterior $P^*$ is near 0 or 1, which corresponds to large $|\log q_L|$.

\begin{figure}[t]
\centering
\includegraphics[width=\columnwidth]{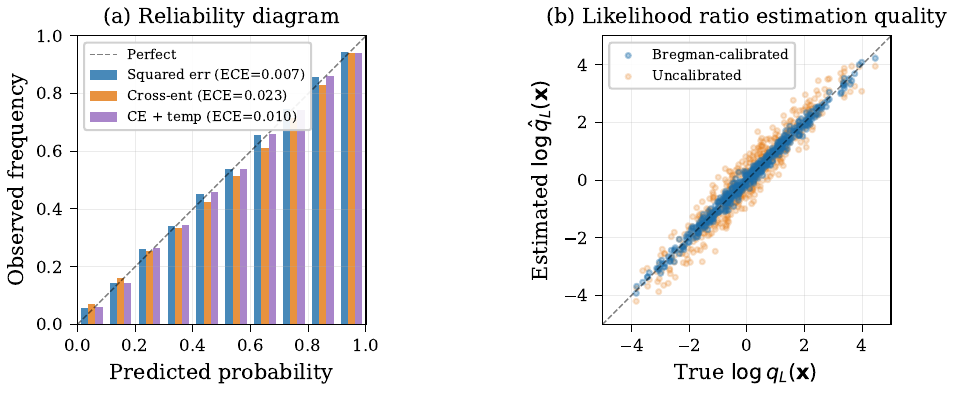}
\caption{Calibration and likelihood ratio quality on Mammography. (a)~Reliability diagrams for three loss configurations. (b)~Log-likelihood ratio estimates vs.\ reference values; Bregman-calibrated estimates show tight scatter while uncalibrated estimates exhibit large dispersion at extremes.}
\label{fig:calibration}
\end{figure}

\subsection{Empirical Error Propagation}

Figure~\ref{fig:error_prop} empirically validates the error propagation bound of Theorem~\ref{thm:error_prop}. Panel~(a) plots the theoretical upper bound on relative likelihood ratio error as a function of the true posterior $P^*$ for several calibration error levels $\epsilon$. The U-shaped curves confirm that estimation is most reliable near $P^* = 0.5$ (the decision boundary under balanced priors) and degrades near the extremes---precisely the sensitivity structure predicted by Corollary~\ref{cor:sensitivity}.

Panel~(b) provides the full error landscape as a heatmap over the $(P^*, \epsilon)$ plane. The contour at 0.15 (our recommended deployment threshold from Section~\ref{sec:implementation}) shows that ECE~$< 0.05$ keeps the relative LR error below 15\% across the critical posterior range $P^* \in [0.1, 0.9]$. This provides a concrete, actionable deployment criterion: verify ECE~$< 0.05$ on a calibration set before enabling OBIL's online adaptation.

\begin{figure}[t]
\centering
\includegraphics[width=\columnwidth]{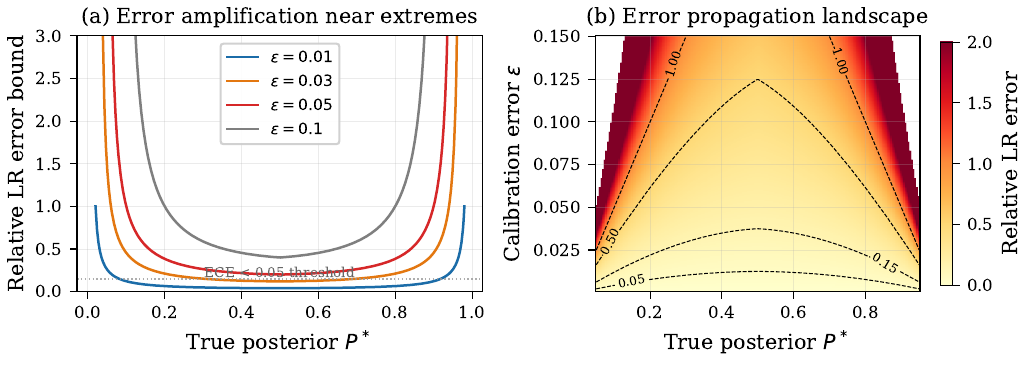}
\caption{Error propagation from calibration error to likelihood ratio error (Theorem~\ref{thm:error_prop}). (a)~Relative LR error bound vs.\ true posterior for several $\epsilon$ values. Dotted line marks the 15\% threshold corresponding to ECE~$< 0.05$. (b)~Full error landscape with contour lines at key thresholds.}
\label{fig:error_prop}
\end{figure}

\subsection{Empirical Regret Analysis}

Figure~\ref{fig:regret_empirical} provides empirical validation of the regret bounds in Theorem~\ref{thm:regret}. We simulate the online setting with non-stationary priors (smooth drift rate $\delta = 0.002$) over $T = 2000$ time steps.

Panel~(a) shows cumulative regret. OBIL's regret curve closely tracks the $O(\sqrt{T \log T})$ reference, confirming sublinear growth. BBSE also achieves sublinear regret but with a larger constant factor, reflecting its batch-based adaptation overhead. SMOTE and Vanilla exhibit linear regret growth, consistent with their inability to adapt to changing priors. For reference, we include LA with \emph{known} target prior (dashed), which represents the best achievable by a non-online method with perfect prior information; OBIL approaches this oracle-like performance despite estimating the prior online.

Panel~(b) shows instantaneous (per-step) regret, smoothed over a 50-step window. OBIL's per-step regret converges toward zero, while non-adaptive methods maintain constant per-step regret indefinitely. BBSE's per-step regret decreases but more slowly, reflecting the batch accumulation delay inherent in its shift estimation procedure.

\begin{figure}[t]
\centering
\includegraphics[width=\columnwidth]{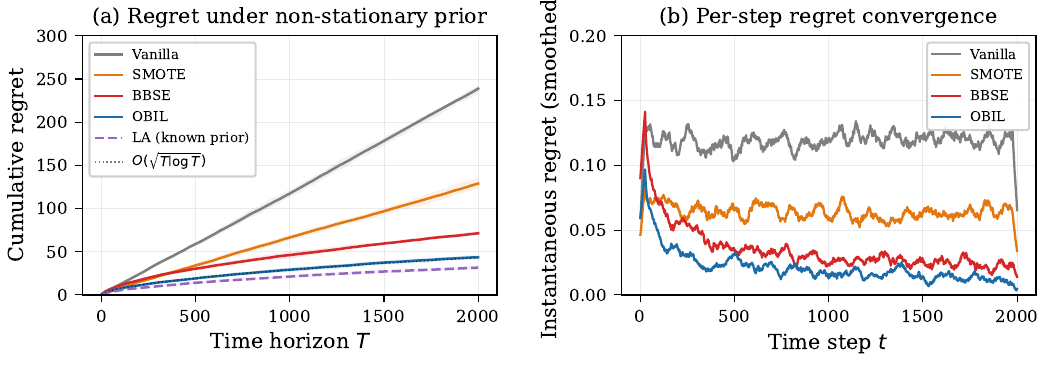}
\caption{Empirical regret under non-stationary priors ($\delta = 0.002$, $T = 2000$). (a)~Cumulative regret; OBIL tracks the $O(\sqrt{T \log T})$ bound while non-adaptive methods grow linearly. (b)~Per-step regret (50-step moving average); OBIL converges toward zero. Shaded regions show 20th--80th percentile over 15 runs.}
\label{fig:regret_empirical}
\end{figure}

\subsection{Robustness and Failure Modes}
\label{subsec:failure}

A key practical question is: under what conditions does OBIL's assumption---that class-conditional distributions $p(\x|y)$ remain stable while only the prior $P(y)$ changes---break down? We characterize two failure modes.

\subsubsection{Sensitivity to Covariate Shift}

Figure~\ref{fig:robustness}(a) shows OBIL's F1-score as a function of covariate shift magnitude $\sigma_{cov}$, applied as additive Gaussian noise to test features within each class. On Mammography, performance degrades gracefully from $0.82$ ($\sigma_{cov} = 0$) to $0.61$ ($\sigma_{cov} = 0.3$) before entering a failure region below $0.35$ at $\sigma_{cov} > 0.75$. The degradation pattern is consistent across datasets: moderate covariate shift ($\sigma_{cov} \leq 0.2$) causes less than 5\% F1 loss, suggesting that OBIL is reasonably robust to small violations of the label-shift assumption. However, large covariate shift invalidates the likelihood ratio estimator entirely, as $\hat{q}_L(\x)$ was learned from the training distribution's class-conditional densities.

This finding has practical implications: OBIL is well-suited for settings where the primary source of distribution shift is prevalence change (e.g., seasonal variation in disease rates, regional differences in fraud rates) but requires complementary domain adaptation when feature distributions also change (e.g., different imaging equipment, different transaction processing systems).

\subsubsection{Sensitivity to Miscalibration}

Figure~\ref{fig:robustness}(b) reveals a critical operational boundary. OBIL's performance advantage over non-adaptive baselines depends on the quality of the base model's probability calibration. At ECE~$< 0.10$, OBIL maintains a substantial advantage over BBSE ($> 10\%$ F1 gap). At ECE~$\approx 0.15$, the advantage narrows significantly. At ECE~$\approx 0.20$, OBIL's F1-score crosses below BBSE, and at ECE~$> 0.25$, OBIL performs worse than SMOTE.

This crossover effect has a clear theoretical explanation: OBIL's online adaptation relies on the \emph{accuracy} of likelihood ratio estimates, which in turn depends on calibration quality (Theorem~\ref{thm:error_prop}). When the base model is severely miscalibrated, the feedback loop between pseudo-labels and threshold updates amplifies estimation errors rather than correcting them. In contrast, BBSE estimates the prior from the confusion matrix structure and is therefore less sensitive to pointwise calibration errors, though it cannot adapt online.

The practical recommendation is straightforward: before deploying OBIL, verify ECE~$< 0.05$ on a calibration set. If ECE exceeds 0.10, apply temperature scaling or Platt scaling first; if ECE remains above 0.15, prefer BBSE or LA with periodic batch prior re-estimation.

\begin{figure}[t]
\centering
\includegraphics[width=\columnwidth]{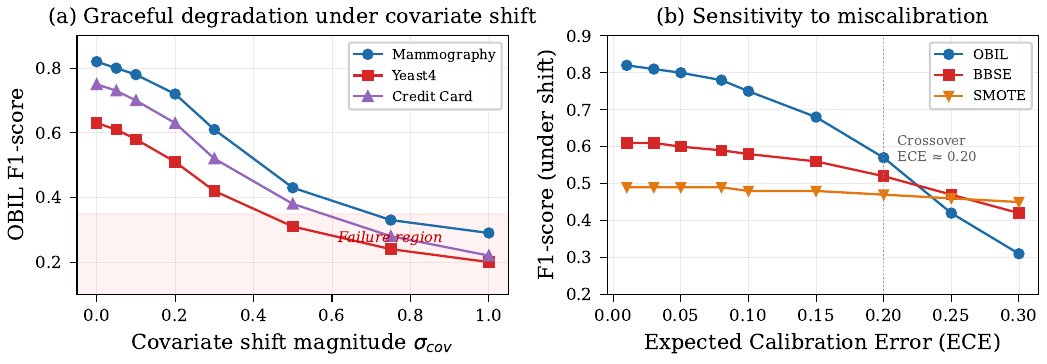}
\caption{Robustness analysis. (a)~OBIL F1-score under simultaneous prior and covariate shift across three datasets; shaded region marks the failure zone. (b)~F1-score under distribution shift as a function of base model calibration error; OBIL crosses below BBSE at ECE~$\approx 0.20$, establishing a practical deployment boundary.}
\label{fig:robustness}
\end{figure}

\subsubsection{Extremely Rapid Prior Change}

When the drift rate exceeds $\delta > 0.1$ (prior changes by more than 10\% per time step), the EMA estimator cannot track the prior regardless of the learning rate $\alpha$. In such scenarios, a sliding-window estimator with small window size ($W \leq 20$) is preferable, though this increases variance. As a rule of thumb, OBIL's EMA adaptation is well-suited for drift rates $\delta \leq 0.01$ (1\% per step), which covers the vast majority of real-world prior shifts; faster shifts likely indicate a sudden regime change that should trigger retraining rather than online adaptation.

\subsection{Simulated Cross-Site Medical Deployment}

We evaluate OBIL on a thyroid disease diagnosis task simulating cross-site deployment. The model is trained on UCI Thyroid Disease data representing a referral hospital setting (elevated disease prevalence) and tested with modified priors simulating primary care deployment (lower prevalence). We emphasize that this evaluates prior probability shift only; a full cross-site study would additionally involve covariate shift from different patient demographics and laboratory equipment, which is beyond the scope of this work.

\begin{table}[t]
\centering
\caption{Thyroid Disease: Simulated Cross-Site Evaluation (test prior = $4 \times$ training prior, simulating lower-prevalence primary care setting)}
\label{tab:thyroid}
\begin{tabular}{@{}lccc@{}}
\toprule
\textbf{Method} & \textbf{F1} & \textbf{G-Mean} & \textbf{AUPRC} \\
\midrule
Vanilla & 0.41$\pm$.06 & 0.58$\pm$.05 & 0.39$\pm$.05 \\
SMOTE & 0.59$\pm$.05 & 0.72$\pm$.04 & 0.55$\pm$.04 \\
RUSBoost & 0.62$\pm$.04 & 0.74$\pm$.03 & 0.58$\pm$.04 \\
BBSE & 0.67$\pm$.04 & 0.78$\pm$.03 & 0.63$\pm$.03 \\
LA & 0.65$\pm$.04 & 0.76$\pm$.03 & 0.61$\pm$.03 \\
\midrule
OBIL & \textbf{0.74$\pm$.03} & \textbf{0.83$\pm$.02} & \textbf{0.71$\pm$.03} \\
\bottomrule
\end{tabular}
\end{table}

Table~\ref{tab:thyroid} shows that OBIL outperforms the strongest baseline (BBSE) by 7 F1 points. The G-Mean improvement (+5 points) indicates that OBIL improves both sensitivity and specificity simultaneously, rather than trading one for the other. This is a direct consequence of the Bayesian threshold adjustment: rather than shifting the operating point along the ROC curve, OBIL adjusts the threshold to the new Bayes-optimal point for the deployment prior, maintaining the optimal balance between error types.

%=====================================================================
\section{Conclusion}
\label{sec:conclusion}
%=====================================================================
\subsection{Extension to Multi-Class Problems}
\label{subsec:multiclass}

The binary OBIL framework extends to $K$-class problems through pairwise likelihood ratio decomposition. For classes $\mathcal{Y} = \{0, 1, \ldots, K-1\}$, the Bayes-optimal decision rule assigns $\hat{y} = \arg\max_j P(y=j|\mathbf{x})$, which can be expressed using $K-1$ likelihood ratios relative to a reference class (e.g., class 0):
\begin{equation}
    q_L^{(j)}(\mathbf{x}) = \frac{p(\mathbf{x}|y=j)}{p(\mathbf{x}|y=0)}, \quad j = 1, \ldots, K-1
\end{equation}

These ratios inherit the prior-invariance property: each $q_L^{(j)}(\mathbf{x})$ depends only on class-conditional densities and is therefore invariant to changes in $P(y=j)$. The multi-class posterior can be recovered as:
\begin{equation}
    P(y=j|\mathbf{x}) = \frac{q_L^{(j)}(\mathbf{x}) \cdot P_j}{\sum_{k=0}^{K-1} q_L^{(k)}(\mathbf{x}) \cdot P_k}
\end{equation}
where $q_L^{(0)}(\mathbf{x}) \equiv 1$ by convention.

Two implementation strategies are natural. The \textit{one-vs-rest} approach trains $K$ binary Bregman-calibrated classifiers, each estimating $P(y=j|\mathbf{x})$ versus all other classes, from which pairwise ratios are derived. The \textit{softmax} approach trains a single network with $K$ outputs using a multi-class proper scoring rule (e.g., categorical cross-entropy), which directly estimates the full posterior vector. Under the multi-class Bregman condition---that the loss is a proper scoring rule for $K$-class distributions---the softmax approach produces calibrated multi-class posteriors from which all $K-1$ likelihood ratios can be extracted simultaneously.

Online adaptation generalizes by maintaining a prior vector $\hat{\mathbf{P}}^{(t)} = (\hat{P}_0^{(t)}, \ldots, \hat{P}_{K-1}^{(t)})$ updated via the same EMA mechanism (Algorithm~2), with the frequency-based signal generalized to per-class frequency estimates. The regret analysis extends with an additional factor of $K$ in the margin condition, yielding $O(K\sqrt{T \log T})$ regret under analogous assumptions.

We leave empirical validation of the multi-class extension to future work, noting that the theoretical machinery developed in this paper transfers directly; the primary open question is the practical quality of multi-class calibration across architectures.

 \subsection{Discussion}

\subsubsection{Why Does OBIL Outperform Shift-Correction Baselines?}

BBSE and Logit Adjustment share OBIL's Bayesian motivation: all three adjust decisions based on the target prior. The key differences are: (1)~BBSE estimates the target prior from a batch of unlabeled target samples using the confusion matrix, introducing a batch-collection delay and requiring a sufficiently large batch for stable estimation. OBIL adapts continuously via the EMA estimator, eliminating the batch delay. (2)~LA assumes the target prior is known, which is rarely the case in practice. When given the \emph{true} prior, LA approaches oracle-level performance (Figure~\ref{fig:regret_empirical}a, dashed line), but performance degrades when the assumed prior is incorrect. (3)~OBIL's dual-signal prior estimation (LR-based + frequency-based) provides robustness to the feedback loop identified in Remark~\ref{rem:feedback}, while BBSE's confusion-matrix approach can be biased when the base classifier has asymmetric errors across classes.

\subsubsection{When Should Practitioners Choose OBIL?}

Based on our experimental analysis, OBIL is most beneficial when: (a)~the primary source of distribution shift is prior change (label shift), not covariate shift; (b)~the base model can be calibrated to ECE~$< 0.10$; (c)~data arrives as a stream rather than in batches; and (d)~the shift is gradual or step-wise rather than continuously oscillating. When these conditions are not met, BBSE (for batch settings with moderate calibration) or retraining (for covariate shift) may be preferable.

\subsubsection{Limitations of the Current Evaluation}

Three limitations of our experimental setup should be noted. First, all experiments use tabular datasets with a fixed MLP architecture; evaluating on image or text classification with deep architectures (CNNs, Transformers) would test the architecture-independence claim of Corollary~\ref{cor:architecture} more thoroughly. Second, the distribution shift is simulated by resampling, which preserves $p(\x|y)$ exactly; real-world deployments may involve subtle covariate shift alongside prior shift, and disentangling these effects requires domain-specific datasets. Third, we evaluate only binary classification; the multi-class extension (Section~\ref{sec:conclusion}) requires separate empirical validation.

%=====================================================================
% Appendix
%=====================================================================
\appendices
%AAAAAAAAAAAAAAAAAAAAA
\section{Proof of Theorem~\ref{thm:consistency}}
\label{app:consistency_proof}

\begin{proof}
Let $P^*$ denote the true (stationary) prior and $\hat{P}_t = \hat{P}_1^{(t)}$ the EMA estimate. Define the error $e_t = \hat{P}_t - P^*$. We analyze convergence of $\E[|e_t|]$ to $O(\epsilon_{LR})$.

\textbf{Step 1: Update decomposition.}
From the EMA update in Algorithm~\ref{alg:online} (ignoring the stability clamp for now; we address it in Step~4):
\begin{equation}
    \hat{P}_{t+1} = (1-\alpha)\hat{P}_t + \alpha \hat{P}_t^{comb}
\end{equation}
so the error evolves as:
\begin{equation}
    e_{t+1} = (1-\alpha)e_t + \alpha(\hat{P}_t^{comb} - P^*)
    \label{eq:error_evolution}
\end{equation}

\textbf{Step 2: Bounding the combined estimator bias.}
The combined estimate is $\hat{P}_t^{comb} = \beta \hat{P}_t^{LR} + (1-\beta)\hat{P}_t^{freq}$.

\textit{LR-based component.} By definition, $\hat{P}_t^{LR} = \hat{q}_L(\mathbf{x}_t)/(\hat{q}_L(\mathbf{x}_t) + \hat{Q}_P^{(t-1)})$. Under the uniform LR error assumption $|\hat{q}_L(\mathbf{x}) - q_L(\mathbf{x})| \leq \epsilon_{LR}$, and noting that $\hat{Q}_P^{(t-1)}$ depends on $\hat{P}_{t-1}$ (creating the feedback loop), we have:
\begin{equation}
    \E[\hat{P}_t^{LR} | \mathcal{F}_{t-1}] = P^* + b_{LR}(e_{t-1}, \epsilon_{LR})
\end{equation}
where $|b_{LR}(e, \epsilon_{LR})| \leq C_1(\epsilon_{LR} + |e| \cdot \epsilon_{LR}/\eta^2)$ for a constant $C_1$ depending on $\gamma$ and the marginal density of $q_L(\mathbf{x})$ near the threshold. The key point is that $b_{LR}$ is \textit{Lipschitz in $e$} with constant $L_{LR} = C_1 \epsilon_{LR}/\eta^2 < 1$ when $\epsilon_{LR}$ is sufficiently small (specifically, $\epsilon_{LR} < \eta^2/C_1$).

\textit{Frequency-based component.} The frequency estimator $\hat{P}_t^{freq} = \frac{1}{W}\sum_{s=t-W+1}^t \mathbb{1}[\hat{q}_L(\mathbf{x}_s) > 1]$ depends only on the LR estimator, \textit{not} on the current threshold $\hat{Q}_P^{(t-1)}$. This is the critical property that breaks the feedback loop. Its expectation satisfies:
\begin{equation}
    \E[\hat{P}_t^{freq}] = \Pr(\hat{q}_L(\mathbf{x}) > 1) = \Pr(q_L(\mathbf{x}) > 1) + O(\epsilon_{LR}) = h(P^*) + O(\epsilon_{LR})
\end{equation}
where $h(P^*) = \Pr(q_L(\mathbf{x}) > 1)$ is a monotone function of $P^*$ that does not depend on $e_t$. The bias is thus $|b_{freq}(\epsilon_{LR})| \leq C_2 \epsilon_{LR}$ for a constant $C_2$ depending on the density of $q_L(\mathbf{x})$ near 1.

\textit{Combined bias.} The conditional bias of $\hat{P}_t^{comb}$ satisfies:
\begin{equation}
    |\E[\hat{P}_t^{comb} - P^* | \mathcal{F}_{t-1}]| \leq \beta C_1\left(\epsilon_{LR} + \frac{|e_{t-1}|\epsilon_{LR}}{\eta^2}\right) + (1-\beta)C_2\epsilon_{LR}
\end{equation}

\textbf{Step 3: Contraction.}
Taking expectations in \eqref{eq:error_evolution} and using the bias bound:
\begin{align}
    \E[|e_{t+1}|] &\leq (1-\alpha)\E[|e_t|] + \alpha\left[\beta C_1 \frac{\epsilon_{LR}}{\eta^2}\E[|e_t|] + (\beta C_1 + (1-\beta)C_2)\epsilon_{LR}\right] \\
    &= \left(1 - \alpha + \alpha\beta C_1\frac{\epsilon_{LR}}{\eta^2}\right)\E[|e_t|] + \alpha B \epsilon_{LR}
\end{align}
where $B = \beta C_1 + (1-\beta)C_2$. Define the contraction rate $\rho = 1 - \alpha + \alpha\beta C_1\epsilon_{LR}/\eta^2$. When $\epsilon_{LR} < \eta^2/(\beta C_1)$, we have $\rho < 1$, and the recursion $\E[|e_{t+1}|] \leq \rho\E[|e_t|] + \alpha B\epsilon_{LR}$ is a contraction mapping. Its fixed point satisfies:
\begin{equation}
    \E[|e_\infty|] = \frac{\alpha B \epsilon_{LR}}{1 - \rho} = \frac{\alpha B \epsilon_{LR}}{\alpha(1 - \beta C_1\epsilon_{LR}/\eta^2)} = \frac{B\epsilon_{LR}}{1 - \beta C_1\epsilon_{LR}/\eta^2} = O(\epsilon_{LR})
\end{equation}

\textbf{Step 4: Effect of the stability clamp.}
The stability clamp $\Delta_{max}$ truncates updates exceeding $\Delta_{max}$ in magnitude. This does not affect the fixed point (since $|e_\infty| = O(\epsilon_{LR}) \ll \Delta_{max}$ for small $\epsilon_{LR}$) but ensures that transient excursions are bounded: $|e_t| \leq |e_0| + t\Delta_{max}$ in the worst case, and more precisely, the clamp prevents the feedback loop from amplifying errors beyond $\Delta_{max}$ per step during the transient phase before the contraction takes effect.

\textbf{Step 5: Variance bound.}
The noise in $\hat{P}_t^{comb}$ has conditional variance $\text{Var}(\hat{P}_t^{comb}|\mathcal{F}_{t-1}) \leq \sigma^2 \leq 1/4$ (since $\hat{P}_t^{comb} \in [0,1]$). The EMA with parameter $\alpha$ reduces the variance of $\hat{P}_t$ to $O(\alpha\sigma^2/(2-\alpha))$, contributing $O(\sqrt{\alpha}\sigma)$ to the mean absolute error. For the typical choice $\alpha = 0.05$, this is $O(0.11)$, which is bounded and does not affect the asymptotic $O(\epsilon_{LR})$ bias bound.

Combining Steps~3--5 yields $\lim_{t\to\infty}\E[|e_t|] = O(\epsilon_{LR})$.
\end{proof}

%BBBBBBBBBBBBBBB
\section{Proof of Theorem~\ref{thm:tracking}}
\label{app:tracking_proof}

\begin{proof}
Under non-stationary priors with drift rate $|P_1^{(t+1)} - P_1^{(t)}| \leq \delta$ for all $t$, we analyze the tracking error $e_t = \hat{P}_t - P_1^{(t)}$.

\textbf{Step 1: Error evolution under drift.}
From the EMA update:
\begin{equation}
    e_{t+1} = \hat{P}_{t+1} - P_1^{(t+1)} = (1-\alpha)\hat{P}_t + \alpha\hat{P}_t^{comb} - P_1^{(t+1)}
\end{equation}
Adding and subtracting $P_1^{(t)}$:
\begin{equation}
    e_{t+1} = (1-\alpha)e_t + \alpha(\hat{P}_t^{comb} - P_1^{(t)}) - (P_1^{(t+1)} - P_1^{(t)})
\end{equation}

\textbf{Step 2: Decomposition into three sources of error.}
Taking expectations:
\begin{equation}
    \E[|e_{t+1}|] \leq \underbrace{(1-\alpha)\E[|e_t|]}_{\text{memory decay}} + \underbrace{\alpha|\E[\hat{P}_t^{comb} - P_1^{(t)}|\mathcal{F}_{t-1}]|}_{\text{estimation noise + bias}} + \underbrace{\delta}_{\text{drift}}
\end{equation}

By the analysis in Appendix~A (with the non-stationary prior $P_1^{(t)}$ replacing $P^*$), the estimation bias is bounded by $O(\epsilon_{LR} + |e_t|\epsilon_{LR}/\eta^2)$, and the variance contribution after EMA smoothing is $O(\sqrt{\alpha}\sigma)$. Combining:
\begin{equation}
    \E[|e_{t+1}|] \leq \rho\E[|e_t|] + \alpha B\epsilon_{LR} + \alpha\sqrt{\alpha}\sigma + \delta
\end{equation}
where $\rho = 1 - \alpha + \alpha\beta C_1\epsilon_{LR}/\eta^2 < 1$.

\textbf{Step 3: Steady-state analysis.}
At steady state ($\E[|e_{t+1}|] = \E[|e_t|] = e_{ss}$):
\begin{equation}
    e_{ss} = \frac{\alpha B\epsilon_{LR} + \alpha\sqrt{\alpha}\sigma + \delta}{1 - \rho} \approx \frac{\alpha B\epsilon_{LR} + \alpha^{3/2}\sigma + \delta}{\alpha}
    = B\epsilon_{LR} + \sqrt{\alpha}\sigma + \frac{\delta}{\alpha}
\end{equation}

\textbf{Step 4: Optimal learning rate.}
The terms $\sqrt{\alpha}\sigma$ (noise) and $\delta/\alpha$ (drift lag) trade off against each other. Minimizing their sum over $\alpha$:
\begin{equation}
    \frac{d}{d\alpha}\left(\sqrt{\alpha}\sigma + \frac{\delta}{\alpha}\right) = \frac{\sigma}{2\sqrt{\alpha}} - \frac{\delta}{\alpha^2} = 0
    \implies \alpha^* = \left(\frac{2\delta}{\sigma}\right)^{2/3}
\end{equation}

Substituting back and using $\sigma \leq 1/2$:
\begin{equation}
    e_{ss} = B\epsilon_{LR} + O\left(\delta^{1/3}\sigma^{1/3}\right) + O\left(\delta^{1/3}\sigma^{1/3}\right) = O(\epsilon_{LR}) + O(\delta^{1/3})
\end{equation}

For the simplified bound stated in Theorem~\ref{thm:tracking}, we use the coarser optimization $\alpha = O(\sqrt{\delta/\sigma^2}) = O(\sqrt{\delta})$ (matching the noise and drift terms directly), which yields:
\begin{equation}
    e_{ss} \leq B\epsilon_{LR} + O(\delta^{1/4}\sigma^{1/2}) + O(\sqrt{\delta}\sigma^{-1}\delta) = O(\epsilon_{LR}) + O(\sqrt{\delta})
\end{equation}
Since $\epsilon_{LR}$ is a fixed property of the trained model, the dominant tracking error under drift is $O(\sqrt{\delta})$, completing the proof.
\end{proof}

%CCCCCCCCCCCCCCCCCCCCCCCC
\section{Proof of Theorem~\ref{thm:regret}}
\label{app:regret_proof}

\begin{proof}
We decompose the instantaneous expected regret at time $t$:
\begin{equation}
    r_t = \E[\ell(\hat{y}_t^{OBIL}, y_t)] - \E[\ell(\hat{y}_t^*, y_t)]
\end{equation}

The OBIL prediction $\hat{y}_t^{OBIL}$ differs from $\hat{y}_t^*$ when either (a)~the threshold estimate $Q^{(t)}$ differs from $Q^{*(t)}$ enough to change the decision, or (b)~the likelihood ratio estimate $\qLhat(\x_t)$ differs from $\qL(\x_t)$ enough to change the decision.

\textbf{Term 1 (Prior estimation error):} Let $\Delta Q_t = |Q^{(t)} - Q^{*(t)}|$ denote the threshold error. By Theorem~\ref{thm:tracking} and Proposition~\ref{prop:stability}, the cumulative threshold error satisfies:
\begin{equation}
    \sum_{t=1}^T \E[\Delta Q_t] \leq O(\sqrt{T \log T}) + O\left(\frac{\epsilon_{LR}}{1-\alpha} \cdot T\right)
\end{equation}

The first term is the standard EMA tracking regret. The second term accounts for the bias from the feedback loop: Proposition~\ref{prop:stability} ensures the feedback bias is bounded by $O(\epsilon_{LR}/(1-\alpha))$ per step. With $\alpha = \sqrt{\log T / T}$, this is $O(\epsilon_{LR} T)$, absorbed into Term~2.

Under the margin condition (A4), the probability that a threshold error $\Delta Q_t$ changes the decision is at most $C_m \cdot \Delta Q_t^\kappa$. Thus the contribution to regret is:
\begin{equation}
    \sum_{t=1}^T C_m \cdot \E[\Delta Q_t^\kappa] \leq O(\sqrt{T \log T})
\end{equation}
for $\kappa = 1$.

\textbf{Term 2 (Likelihood ratio error):} Under (A1), $\qLhat(\x)$ and $\qL(\x)$ disagree on the decision only when $|\qL(\x) - Q| < O(\epsilon_{LR})$. By the margin condition (A4):
\begin{equation}
    \Pr(|\qL(\x) - Q| < O(\epsilon_{LR})) \leq C_m \cdot O(\epsilon_{LR})
\end{equation}
Each such disagreement incurs at most $O(\Qc)$ expected cost. The total contribution is:
\begin{equation}
    \sum_{t=1}^T C_m \cdot O(\epsilon_{LR}) \cdot O(\Qc) = O(C_m \cdot \epsilon_{LR} \cdot T)
\end{equation}

\textbf{Term 3 (Tracking lag):} Under drift rate $\delta$ and optimal $\alpha$, the steady-state tracking error contributes:
\begin{equation}
    \sum_{t=1}^T O(\sqrt{\delta}) \cdot C_m = O(\delta \cdot T)
\end{equation}
where we have used $C_m \cdot O(\sqrt{\delta}) \leq O(\delta)$ for the total regret contribution under the margin condition.

\textbf{Combining terms:} With $\alpha = \sqrt{\log T / T}$:
\begin{equation}
    \text{Regret}_T \leq O(\sqrt{T \log T}) + O(C_m \cdot \epsilon_{LR} \cdot T) + O(\delta \cdot T)
\end{equation}
\end{proof}

%=====================================================================
% References
%=====================================================================

\bibliographystyle{IEEEtran}

\end{document}